\documentclass{ecai} 

\usepackage{latexsym}
\usepackage{amssymb}
\usepackage{amsmath}
\usepackage{amsthm}
\usepackage{booktabs}
\usepackage[inline]{enumitem}
\usepackage{graphicx}
\usepackage{color}

\usepackage{todonotes}
\usepackage{balance} 
\usepackage{caption}
\usepackage{graphicx} 
\usepackage{subcaption} 
\usepackage{algorithm}
\usepackage{algorithmic}
\usepackage{multirow}
\usepackage{amsthm}
\usepackage{caption}
\newtheorem{example}{Example}
\usepackage[breakable,most]{tcolorbox}

\newtcolorbox[auto counter]{conversation}[2][]
  {
   colback=gray!5.5!white,
   colframe=black!65!black, 
   fonttitle=\bfseries,
   fontupper=\sffamily\fontsize{7.5pt}{10.5pt}\selectfont,
   colbacktitle=gray!5.5!white, enhanced,
   coltitle=black,
   attach boxed title to top left={yshift=-2.5mm, xshift=4mm},
   title=#2, boxrule=0.3pt, #1,
   rounded corners, arc=2mm,
   boxed title style={boxrule=0.3pt, rounded corners, arc=2mm},
   label type=table
   }

\usepackage{amsmath}
\DeclareMathOperator*{\argmax}{arg\,max}

\usepackage{cleveref}

\newtheorem{theorem}{Theorem}
\newtheorem{lemma}[theorem]{Lemma}

\newtheorem{proposition}[theorem]{Proposition}

\newcommand{\BibTeX}{B\kern-.05em{\sc i\kern-.025em b}\kern-.08em\TeX}

\begin{document}


\begin{frontmatter}


\paperid{785} 


\title{Balancing Act: Prioritization Strategies for LLM-Designed
Restless Bandit Rewards}


\author[1]{\fnms{Shresth}~\snm{Verma}\thanks{Equal contribution.}}
\author[2]{\fnms{Niclas}~\snm{Boehmer}$^*$}
\author[1]{\fnms{Lingkai}~\snm{Kong}} 
\author[1]
{\fnms{Milind}~\snm{Tambe}} 

\address[1]{Harvard University, USA}
\address[2]{Hasso Plattner Institute, Germany}


\begin{abstract}
\textcolor{black}{LLMs are increasingly used to design reward functions based on human preferences in multiagent Reinforcement Learning (RL)}. We focus on LLM-designed rewards for Restless Multi-Armed Bandits, a framework for allocating limited resources among agents. In applications such as public health, this approach empowers grassroots health workers to tailor automated allocation decisions to community needs. 
In the presence of multiple agents, altering the reward function based on human preferences can impact subpopulations very differently, leading to complex tradeoffs and a multi-objective resource allocation problem.
We are the first to present a principled method termed \emph{Social Choice Language Model} for dealing with these tradeoffs for LLM-designed rewards for multiagent planners in general and restless bandits in particular.  
The novel part of our model is a transparent and configurable selection component, called an \textit{adjudicator}, external to the LLM that controls complex tradeoffs via a user-selected social welfare function. 
Our experiments demonstrate that our 
model reliably selects more effective, aligned, and balanced reward functions compared to purely LLM-based approaches. 


\end{abstract}

\end{frontmatter}


\section{Introduction} \label{intro}
Reward functions play a fundamental role in the generation of optimal policies for sequential decision-making via reinforcement learning. 
Previous work has shown that LLMs are an effective tool for designing reward functions that can be guided and customized via human language prompts \cite{DBLP:journals/corr/abs-2310-12931,DBLP:journals/corr/abs-2404-00282,DBLP:conf/iclr/KwonXBS23,DBLP:conf/iclr/XieZWLLZY024,DBLP:journals/corr/abs-2406-01967,DBLP:conf/corl/0003GFKLACEHHIX23,DBLP:journals/corr/abs-2406-01309}. 
\textcolor{black}{Focusing on optimization and planning scenarios, we study the problem of designing high-quality reward functions aligned with human preference prompts in a \emph{multiagent} context, rendering  the underlying problem inherently multi-objective.}
We present a transparent framework around LLMs that constructs effective, aligned, and balanced reward functions for complex human prompts.

We study the reward design problem for restless multi-armed bandits (RMABs), a popular model in multiagent systems for sequentially allocating a limited number of resources to a set of agents \cite{whittle1988restless,nino2023markovian}.
In RMABs, \textcolor{black}{there are multiple, independently evolving agents, with each agent being represented by an individual Markov Decision Process including a reward function}. 
By choosing these reward functions, one can control which agents are more or less likely to receive a resource. 
RMABs have been \textcolor{black}{applied to multiagent problems in various domains}  such as machine maintenance \cite{DBLP:journals/informs/AbbouM19}, anti-poaching \cite{DBLP:conf/atal/QianZKT16}, and healthcare \cite{ayer2019prioritizing,DBLP:journals/aim/VermaSMVGMHTJTT23}. 
In many of them, system organizers have evolving allocation priorities based on agents' features that need to be incorporated into the resource allocation process \cite{deardorff2018strategies,group}. 
For instance, in a healthcare program, 
a healthcare worker might
want to change the allocation policy to prioritize
low-income beneficiaries who are at higher risk or older beneficiaries who have transportation barriers for healthcare access   
\cite{nelson2016disparities,syed2013traveling} 
\textcolor{black}{via the following preference prompt}:
\emph{Prioritize low-income beneficiaries and older beneficiaries}.


\color{black}Unfortunately, handcrafting reward functions is often a challenging and time-consuming task for humans because of the complex relationship between reward functions and policy outcomes \cite{DBLP:journals/corr/abs-2406-01309,DBLP:journals/corr/abs-2404-00282,DBLP:journals/corr/abs-2310-12931}. \color{black}
Further, the multiagent nature of the RMAB problem adds a new twist to the problem of reward design in RL: It becomes fundamentally \emph{multi-objective}.
Consider the above example prompt asking for the prioritization of two subpopulations. 
As these subpopulations may contain different agents, selecting a reward function will most likely involve trading off the interests of the low-income vs.\ older beneficiaries, making this a multi-objective problem. 
If this multi-objective nature is ignored, a selected reward function might heavily favor one of the two groups (e.g., leading to the allocation of many resources to low-income beneficiaries, and no resources to older ones). 

This problem of multi-objective reward function modification even extends beyond Restless Multi-Armed Bandits, to challenges relevant in using Game Theory and AI for security, for instance, in Stackelberg Security Games (SSGs) applications \cite{tambe2011security}. In SSGs, optimal security strategies are highly sensitive to the defined payoff functions, which typically reflect the defender's and attacker's objectives. Using a natural language interface offers defenders to refine and balance these payoff functions when faced with conflicting objectives can act as a powerful tool. This is in contrast with traditional iterative and manual adjustments presented in prior SSG applications. For instance, the ARMOR system deployed at LAX airport for security resource allocation \cite{pita2008deployed} originally included a graphical user interface to enable payoff adjustments, but this interface was rudimentary. A natural language interface, as proposed in our work, could significantly enhance the practical utility of SSGs by allowing security personnel to express nuanced objectives (e.g., "prioritize slightly more air marshals on routes to Tokyo and Paris" using a system like IRIS \cite{tsai2009iris}) directly, and have these preferences be translated into balanced and effective security strategies. This capability addresses a critical need for more flexible and human-centric control over complex security resource allocation.

To our knowledge, we are the first to address the multi-objective nature of LLM-powered reward design in RMABs in particular and multiagent planners in general. 
Closest to our paper is 
the work by 
\citet{DBLP:journals/corr/abs-2402-14807} who proposed a fully LLM-based Decision-Language Model for RMABs to generate and select reward functions (as code) from human language prompts.
However, as argued in \Cref{rw,sec:Prob}, the DLM model is not properly equipped to handle the multi-objective nature of the problem, as the LLM selects functions in an unpredictable, hard-to-control and sometimes (clearly) suboptimal way that does not adequately take into account and balance the different objectives. 

We present a Social Choice Language Model (SCLM) that designs reward functions (as Python code) aligned with complex, multi-objective human language preferences; see \Cref{fig:ov} for an overview of SCLM. 
Our pipeline separates the generation of candidate reward functions in the \emph{generator} from the selection of one function in the \emph{adjudicator}.
For the generator, we use LLM-powered evolutionary search to generate a pool of reward functions \cite{DBLP:journals/corr/abs-2402-14807}.  
In the transparent and customizable adjudicator, we take a new social choice perspective to address the multi-objective nature of our problem: 
We create a scorer component that evaluates the quality of generated reward functions according to the different objectives (e.g., different prioritization requests).
Subsequently, a social welfare function aggregates these ``alignment scores'' to select the best reward function. 
The user can select the social welfare function and thereby has additional control over the preferred trade-off between objectives, e.g., maximizing the summed vs.\ minimum alignment of all objectives. 
We show that SCLM returns high-quality reward functions even if the computed alignment scores are noisy.
In our experiments, we demonstrate that  SCLM leads to the selection of reward functions significantly better aligned with complex, multi-objective prompts. 
Moreover, we also show how it can be used to effectively mitigate the risks of using rewards designed from human prompts: unintended effects for other agents and the ineffective allocation of resources.  
\emph{Overall, SCLM combines the generative power of LLMs to design reward functions with the capabilities of social choice to handle multi-objective decision-making scenarios.}

\section{Related Works}\label{rw}

\paragraph{LLM-enhanced RL}
LLMs have emerged as a powerful tool to enhance RL. Recent work has used LLMs to generate reward functions based on natural language descriptions~\cite{DBLP:journals/corr/abs-2310-12931, xie2024textreward, DBLP:conf/corl/0003GFKLACEHHIX23}. For instance, 
\citet{goyal2019using, carta2022eager, mirchandani2021ella,DBLP:journals/corr/abs-2406-01309}
shape rewards by training
an RL agent to learn and complete intermediate tasks guided by language, yet focusing on very different (non-multiagent)~environments.

The work of \citet{DBLP:journals/corr/abs-2402-14807}  is the first to present a Decision-Language Model for generating reward functions for RMABs from human prompts.
The model performs a form of evolutionary search to find reward functions aligned with the given prompt in two interleaving phases: generation and reflection. 
In the generation phase, an LLM generates a set of reward functions. Based on reward function's performances, in the reflection phase \cite{DBLP:journals/corr/abs-2310-12931,DBLP:conf/nips/ShinnCGNY23}, the LLM selects the function best aligned with the prompt. 
This function is then included in the prompt for the next generation phase or returned. 
In contrast to our work, DLM mixes generation with selection and does not explicitly account for the multi-objective nature of the reward design problem. 
Furthermore, in contrast to our work, they focus on small RMAB instances ($\sim 20$ arms).
Throughout the paper, we will use a slightly modified variant of DLM adjusted to our setting (see  \Cref{app:DLM}) as a baseline. 



\paragraph{Multi-Objective Reinforcement Learning (MORL)}
Research on MORL focuses on learning policies that maximize (and balance between) multiple objective functions, typically via scalarizing the objectives into a single reward function \cite{DBLP:conf/emo/MoffaertDN13} or approximating the Pareto front~\cite{roijers2013survey, van2014multi,kim2025navigating}. 
In the context of multiagent systems, MORL has been used as a method to ensure the fair treatment of the individual agents \cite{jiang2019learning, zimmer2021learning,DBLP:conf/atal/0002PTF23}. 
Closest to ours from these lines of work are the papers by  \cite{DBLP:conf/atal/0002PTF23} and \cite{kim2025navigating}.  \cite{DBLP:conf/atal/0002PTF23} uses ideas from the resource allocation literature to combine multiple objectives into a singular non-linear objective function and focuses on policy learning for such non-linear objective functions. \cite{kim2025navigating} focuses on finding a set of policies that approximate the Pareto front for various sequential planning problems.
However, in contrast to our paper, neither considers reward design, human natural language preference prompts and LLMs.

\textcolor{black}{\paragraph{Inverse Reinforcement Learning (IRL)}
Another area related to this work is inverse reinforcement learning (IRL) \cite{arora2021survey,jain2024irl}, which also focuses on obtaining the reward function. However, IRL aims to recover the underlying reward function that drives an expert’s demonstrated behavior. In other words, IRL seeks to understand why an agent acts in a certain way by observing its actions, rather than directly specifying its goals. In contrast, our approach does not depend on expert demonstrations, as such data is unavailable in our setting. Instead, we explicitly optimize the reward function to guide the agent toward the desired behavior.}

We refer to \Cref{app:RW} for additional related work.



\section{Preliminaries}

An instance of Restless Multi-Armed Bandits (RMAB) is defined by a set of $N$ arms, a time horizon $T$, and a budget $K$. We also refer to arms as agents. 
Each arm $i\in [N]$ is an independently evolving MDP with state space $\mathcal{S}_i$, actions $\mathcal{A}_i=\{0,1\}$, transition function $P_i:\mathcal{S}_i \times \mathcal{A}_i \times \mathcal{S}_i \to \mathbb{R}_{\geq 0}$, and reward function $R_i:\mathcal{S}_i \to \mathbb{R}$.
We refer to $1$ as the \emph{active} action corresponding to pulling the arm (i.e., allocating a resource) and $0$ as the \emph{passive} action corresponding to not pulling the arm.  
We focus on the popular case where each MDP consists of two states, i.e., $\mathcal{S}_i=\{0,1\}$ for all $i\in [N]$, yet our methodology applies to MDPs with arbitrary state spaces.   
We refer to $0$ as the \emph{bad} and $1$ as the \emph{good} state. 
For each step in which an agent is in the good state, they derive a \emph{utility} of $1$, while they derive a utility of $0$ in the bad state. 
Accordingly, agents' \emph{default reward function} $R^*$ is $R^*(s)=s$.
We assume that there is a set of categorical features. Each arm is associated with a value of each feature. 
A \emph{global reward function} is a reward function defined over features, which induces a reward function for each arm by plugging in its feature values~(see~\Cref{ex}).

In each step within the time horizon $T$, the planner observes the state of all arms and decides to pull a subset of at most $K$ arms. 
As solving the RMAB problem optimally is  computationally intractable  \cite{papadimitriou1994complexity}, we make use of the very popular state-dependent Whittle index \cite{whittle1988restless,nino2023markovian}, which given arms' reward functions tries to quantify for each state of each arm the reward gain achieved from applying the active action to the arm in this state.
In the Whittle index policy $\Pi$, in each step, we compute the Whittle index for each arm (based on its current state) and pull the arms with the $K$ highest Whittle indices. 
We will use it as the solution strategy in the following.

For a global reward function $R$, we write $\Pi(R)$ to denote the Whittle index policy for $R$, i.e., the Whittle index policy for the instance where each agent uses the function $R$ after plugging in their feature values as their reward.
We refer to $\Pi(R^*)$ as the \emph{default policy}.
To assess the quality of a global reward function $R$, we often consider the  \emph{utility feature distribution} for some feature $X$. 
This distribution shows for each value of the feature, the expected utility generated by arms with this feature value under the policy $\Pi(R)$ (see \Cref{fig:fair1} for an example).

\section{Problem Statement \& Challenges}\label{sec:Prob}
We assume that we are given a human-language preference prompt, concatenating one or multiple \emph{preference clauses}. 
Each preference clause specifies a single optimization goal. We explicitly consider three types of preference clauses (yet our methodology extends to arbitrary ones):
\begin{enumerate*}[label=(\roman*)]
\item Give priority to agents with certain  feature values, i.e., increase the utility generated by these agents, 
\item do not shift the utility distribution for some feature, and  
\item maximize the summed utility generated by all agents.
\end{enumerate*}
We mostly focus on the first type and refer to them as \emph{prioritization clauses and prompts}.
A preference prompt is a set $P=\{p_1,p_2,\dots\}$ of the involved preference clauses. 
We call a prompt $P$ \emph{singular} if $|P|=1$ and \emph{composite} otherwise; our focus is on the latter.
We can influence the utility agents generate by selecting a single global reward function (inducing reward functions $(R_i)_{i\in [n]}$ for all agents).  
\begin{example}\label{ex}
    Consider an RMAB instance with three binary features $A$, $B$, and $C$. A preference prompt $P$ could be ``Prioritize agents with $A=0$ and prioritize agents with $B=1$'', i.e., $P=\{$``prioritize agents with $A=0$''$,$   ``prioritize agents with $B=1$'' $\}$. 
    Two possible global reward functions for the prompt are $R'(s)=s\cdot (1-A)\cdot B$ and $R''(s)=s\cdot(1-A)+s\cdot B$.
    For function $R''$, the reward of an agent $i$ with $A=0$ and $B=1$ is $R_i(s)=2s$, while the reward  of an agent $j$ with $A=1$ and $B=1$ is $R_j(s)=s$. Selecting $R''$, agent $i$ is more likely to receive a resource than agent $j$, as the good state contributes more reward for $i$. 
\end{example}

We want to design a single global reward function that is ``well-aligned'' with all clauses of the given human-language preference prompt.
However, as clauses can contradict each other, perfect alignment with all clauses becomes impossible. For instance, if a prompt requests the prioritization of two fully disjoint subpopulations, each resource will only benefit one of the two.
When picking the reward function, we need to carefully balance the interests of the two groups of agents against each other.
Generally, in the presence of multiple agents and limited resources, each clause can be viewed as a separate independent objective that we want to optimize, rendering this a multi-objective problem. 

\begin{figure}[t!]
\centering
        \begin{subfigure}{0.23\textwidth}
            \includegraphics[width=1.1\textwidth]{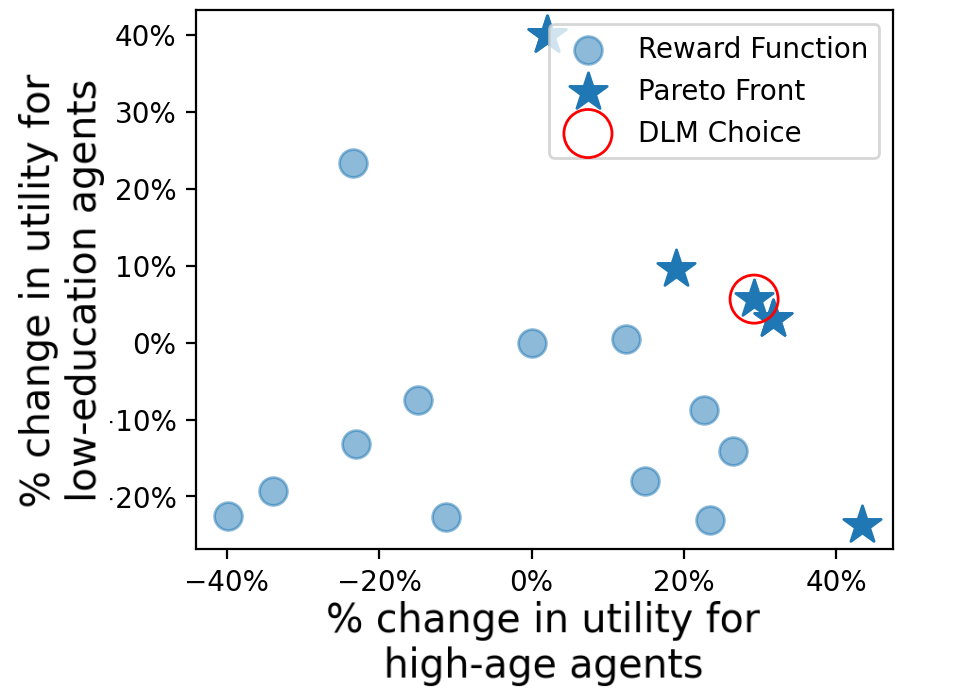}
            \caption{Prompt: ``Prioritize agents with old age and agents with low education''}
            \label{fig:sub1}
        \end{subfigure}
        \hfill
        \begin{subfigure}{0.23\textwidth}
            \includegraphics[width=1.1\textwidth]{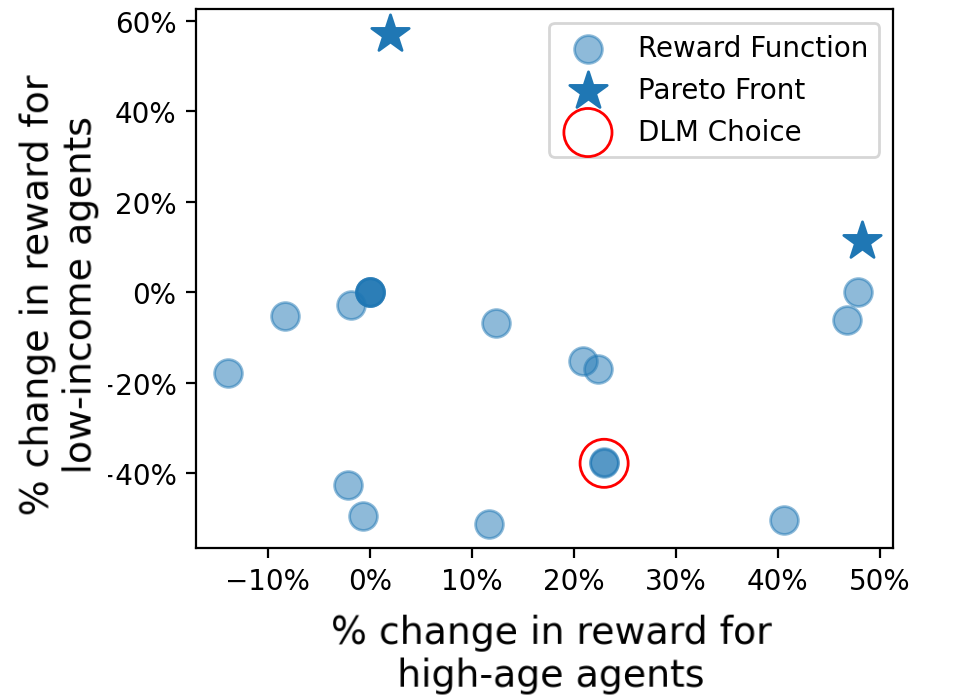}
            \caption{Prompt: ``Prioritize agents with high age and agents with low income''}
            \label{fig:sub4}
        \end{subfigure} \vspace{-0.5cm}
        \caption{Tradeoffs between prioritization clauses.} 
        \label{fig:pareto_tradeoff}
\end{figure}

To illustrate the tradeoff decisions we face between different clauses when selecting a reward function, in  \Cref{fig:pareto_tradeoff}, we show two instances from our experiments for a prompt consisting of two prioritization clauses. 
Every point represents a LLM-designed reward function. 
The $x$ and $y$ axes represent the quality of the reward function from the perspective of the two prioritized subgroups where higher percentage values indicate more benefits (see \Cref{TradeoffExp} for details). 
Reward functions marked with stars lie on the Pareto fronts (no other available function dominates them). 

In our experiments, we observe that the  DLM model from previous work picks functions from very different parts of the Pareto frontier, potentially clearly prioritizing one subgroup over another (see \Cref{fig:sub1} for an example).
In many other instances, it also picks suboptimal functions, i.e., functions that do not lie on the frontier, that may even harm one of the subgroups while strongly benefiting the other (see \Cref{fig:sub4}).
This highlights the risks (and shortcomings of DLM) in not accounting for the multi-objective nature of the problem, as it picks reward functions that are inefficient (i.e., dominated) and unfair (i.e., heavily favoring one clause over the other). 
\color{black}

Another shortcoming of DLM are unintended utility shifts.
\begin{figure}[t]
        \centering
        \begin{subfigure}{0.23\textwidth}\includegraphics[width=\textwidth]{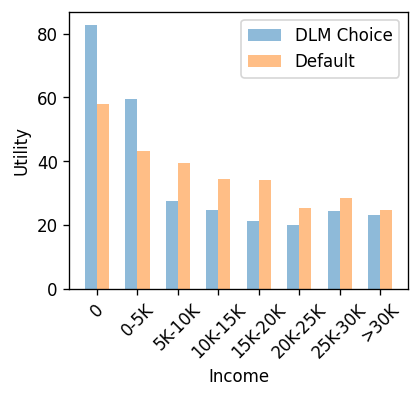}
\caption{Income feature}
\label{fig:fair1}
        \end{subfigure}
        \begin{subfigure}{0.23\textwidth}\includegraphics[width=\textwidth]{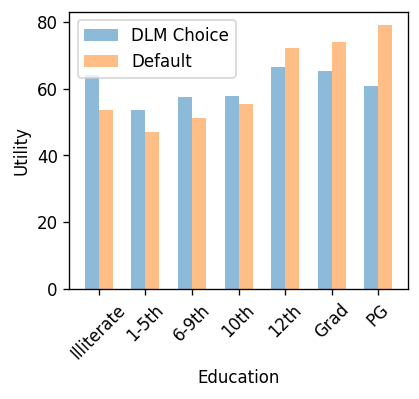}\caption{Education feature}\label{fig:fair3}
        \end{subfigure}  \vspace{-0.3cm}
        \caption{\color{black}Utility feature distributions for default reward function (orange) and reward function returned for prompt ``Prioritize  agents with low income'' (blue) by DLM baseline. $x$-axis depicts feature value and $y$-axis total utility generated by agents with this value.\color{black}}
        \label{fig:fair}
    \end{figure}
Moving from the default reward function to a reward function aligned with a given (prioritization) prompt causes shifts within the distribution of resources and utility.  
Due to correlations between features, 
this change might lead to unintended utility shifts for features not specified in the prompt.  \Cref{fig:fair} shows an example of this from our experiments. 
We present the utility feature distribution for the two features \emph{income} and \emph{education} for two reward functions: 
The reward function selected by DLM for the prompt ``Prioritize agents with low income''  (orange) and the default reward (blue). 
While the utility generated by low-income agents increases when moving from the default to the customized reward function, the utility generated by highly educated agents decreases, a side-effect the end-user might be unaware of and that might conflict with their allocation goals. In our proposed approach, we are able to account for this issue by incorporating the prevention of unintended utility shifts as a tradeoff dimension.
\color{black}

Thus, through the Social Choice Language Model, \emph{our goal is to create a model that handles multiple tradeoffs posed by composite  ``multi-objective'' prompts in a principled, transparent, and customizable fashion and outputs a single effective and fairly aligned global reward function. }

\section{Social Choice Language Model (SCLM)}\label{sec:SCLM}

\begin{figure*}[t]
    \centering
    \includegraphics[width=0.85\linewidth]{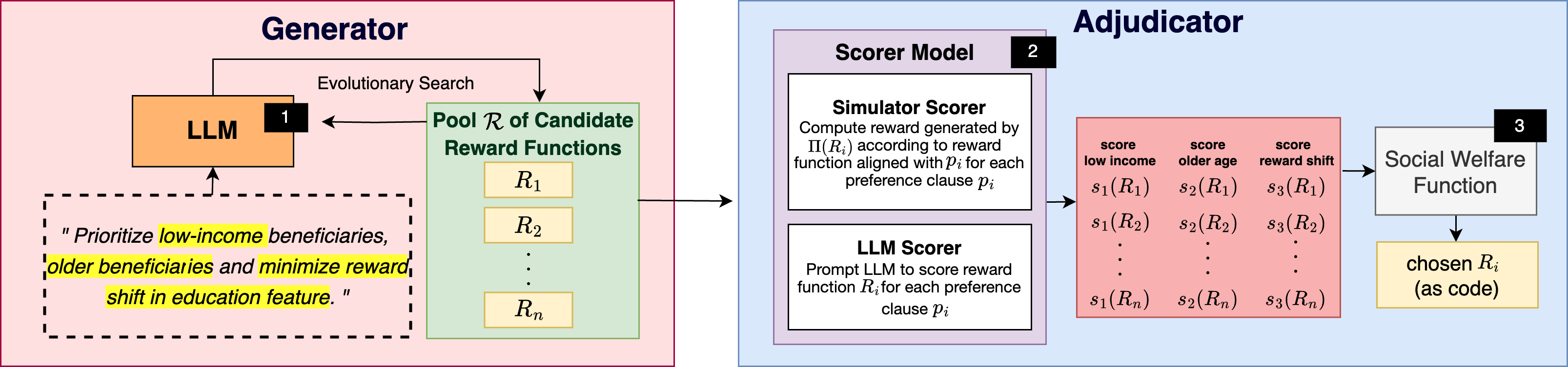}
    \vspace{5pt}
    \caption{Overview of SCLM. In step 1, preference prompt is passed to the generator, which performs an evolutionary search to create a pool $\mathcal{R}$ of candidate reward functions. In step 2, these functions are passed to the adjudicator where a scorer model (e.g., the simulator or LLM scorer) computes the alignment scores. In step 3, a user-defined social welfare function selects a reward function based on the alignment scores.}
    \label{fig:ov}
    \vspace{-0.5cm}
\end{figure*}



We propose a Social Choice Language Model to generate rewards from human language composite preference prompts (see \Cref{fig:ov} for a visualization). 
Separating the generation and selection of reward functions, the model consists of two sequential components. 
The LLM-powered \emph{generator} generates a set of candidate reward functions. 
Subsequently, taking a social-choice-inspired viewpoint, the \emph{adjudicator} selects a reward function from the pool to be returned to the user in two steps: 
First, a scorer model computes an alignment score for each reward function with each prioritization clause (i.e., we judge each reward function from the perspective of all relevant ``objectives''). 
Second, a user-defined social welfare function aggregates these scores into a ``winning'' candidate reward function.
By selecting the social welfare function, the user can control the area of the Pareto frontier from which reward functions get selected. 
While we remark that our model can also be used to tackle multi-objective issues arising when designing rewards in single-agent RL, the details of our components (e.g., the reflection in the generator and the computation of alignment scores) are specific to the multiagent nature of the RMAB problem.

\subsection{Generator} \label{gen}
Given a prompt, our generator creates a set of candidate reward functions (as Python code) via a variant of evolutionary search following \citet{DBLP:journals/corr/abs-2402-14807}:
We proceed in multiple steps. 
First, inputting 
the problem description, feature descriptions and the preference prompt, we ask an LLM to generate code for a reward function. We repeat this query $n_p$  times to obtain a set $\mathcal{R}$ of  $n_p$ candidate reward functions. 
Afterwards, for each function $R\in \mathcal{R}$ we compute the utility feature distributions of the policy $\Pi(R)$ induced by the reward function $R$ on the given RMAB instance (via repeatedly simulating the policy on the instance). 
Then, the prompt and the set of candidate reward functions together with the associated utility feature distributions are passed to an LLM, which is asked to select the reward function $R'$ from  $\mathcal{R}$ best aligned with the prompt \cite{DBLP:journals/corr/abs-2310-12931,DBLP:conf/nips/ShinnCGNY23}. 
Now, we repeat the whole process, this time including the selected policy $R'$ as a seed in the reward function generation prompts. 
Once we have executed the process $n_r$ times, we add all generated $n_p\cdot n_r$ candidate reward functions $R$ to the pool $\mathcal{R}$ (see  \Cref{app:DLM} for details).

\subsection{Adjudicator}\label{ad}
The adjudicator selects a reward function from a given pool of candidate reward functions returned by the generator. 
To handle the complex tradeoffs arising within composite prompts and the resulting multi-objective optimization problem, the adjudicator follows a social choice approach. 
Social choice is a discipline at the intersection of economics, philosophy, and mathematics and concerned with 
aggregating the potentially contradicting preferences of a set of voters into a fair compromise alternative from a given candidate set \cite{arrow2010handbook,moulin2004fair}. 
It thus provides a theoretically grounded and well-studied methodology for balancing competing interests. 
In our problem, we can interpret the reward functions as the candidates and the preference clauses in the prompt as the voters with their preferences over the candidates reflecting the reward function's alignment with the clause. 
This view gives rise to the following strategy: 
Given a prompt $P=\{p_1,p_2,\dots,p_\ell\}$, we evaluate each reward function $R\in \mathcal{R}$ from the perspective of each preference clause $p_i$ by computing an (alignment) score $s_i(R)$.  $s_i(R)$ measures the alignment of $\Pi(R)$ with preference clause $p_i$, i.e.,  how much the voter representing $p_i$ ``likes'' the candidate $R$.

In the following, in \Cref{soc:welfare}, we describe how the adjudicator selects the reward function given the scores; in \Cref{sec:align}, we describe how scores get computed; and in \Cref{sec:error}, we present a guarantee on the quality of the selected reward function if computed scores are noisy.


\subsubsection{Selection via Social Welfare Function}\label{soc:welfare}
Social welfare functions select an alternative based on input preferences of voters. The pros and cons of individual social welfare functions have been extensively researched and debated in social choice \cite{arrow2010handbook,rawls2017theory}. 
The \textit{generalized p-mean}, $f_p(\cdot): \mathbb{R}^l_{> 0} \to \mathbb{R}_{>0}$, is a rich class of social welfare functions which we consider in this work. 
It is defined for a given $p \in (-\infty, 1]$ and strictly positive vector $\mathbf{s}(R) = (s_1(R),\dots,s_l(R)) \in \mathbb{R}^l_{> 0}$ as follows:
\begin{equation}
f_p(\mathbf{s}(R)) =
\begin{cases}
\min_{i \in [l]} s_i(R) & \text{if} \ p = -\infty, \\
\displaystyle
\left(\frac{1}{l} \sum_{i \in [l]} s_i(R)^p\right)^{\!1/p}
& \text{if } p \not\in \{-\infty, 0\}, \\[1em]
\displaystyle
\left(\prod_{i \in [l]} s_i(R)\right)^{\!1/l}
& \text{if } p = 0.
\end{cases}
\end{equation}
In our experiments, we consider the three arguably most popular social~welfare~functions, which can be written as a generalized $p$-mean: 
\begin{description}
    \item[Utilitarian ($p=1$)] Return the reward function maximizing the sum of its scores, i.e., $\argmax_{R\in \mathcal{R}} \sum_{i\in [\ell]} s_i(R)$. 
    \item[Nash ($p=0$)] Return the reward function maximizing the product of its scores, i.e.,  $\argmax_{R\in \mathcal{R}} \prod_{i\in [\ell]} s_i(R)$. 
    \item[Egalitarian ($p=-\infty$)] Return the reward function maximizing its minimum score, i.e., $\argmax_{R\in \mathcal{R}} \min_{i\in [\ell]} s_i(R)$.
\end{description}
Selecting the social welfare function gives us control over the tradeoffs between objectives:
By picking the Egalitarian function, we ensure that one clause will not get prioritized over another. 
In contrast, the Utilitarian function prioritizes the summed alignment, allowing for mismatches between clauses; the Nash function strikes a balance between the two functions.\footnote{Note that social welfare functions also allow for assigning a different importance to clauses: 
The user could submit an importance score $w_i$ for each clause $p_i$, which can be easily incorporated in the social welfare function, e.g., the Utilitarian welfare function becomes $\argmax_{R\in \mathcal{R}} \sum_{i\in [\ell]} w_i\cdot s_i(R)$.} 
The adjudicator makes the selection process more transparent, as the different objectives, the selection criterion, and the performance of the candidate reward functions regarding the objectives become explicit.

\subsubsection{Computing Alignment Scores}\label{sec:align}
It remains to describe how the alignment scores $s_i(R)$ are computed. 
We present two general methods to compute alignment scores, which we will use for prioritization clauses. 
Subsequently, we discuss two more customized methods for the prevention of unintended utility shifts or drops in total generated utility. 

\paragraph{Simulator Scorer Model (SCLM-SIM)}
For each preference clause $p_i\in P$, we compute a reward function $R_i$ aligned with $p_i$ by casting it as a singular prompt to the DLM pipeline (see \Cref{app:DLM}). 
For each $R\in \mathcal{R}$, we compute as $s_i(R)$ the expected reward according to reward function $R_i$ produced by policy $\Pi(R)$ (again, we approximate this quantity by running multiple simulations). Accordingly, $s_i(R)$ quantifies the quality of the policy induced by the candidate reward function $R$ from the perspective of $p_i$ (as captured by $R_i$). 
As the scale of the reward functions can vary significantly among preference clauses, we normalize the scores by the performance of the default policy, i.e., we compute $\frac{s_i(R)-s_i(R^*)}{s_i(R^*)}$.

\paragraph{LLM Scorer Model (SCLM-LLM)}
The Simulator Scorer Model assumes access to reward functions capturing individual preference clauses well. 
If no well-aligned reward functions can be obtained, the performance of SCLM-SIM can deteriorate because it can become  noisy.
Another disadvantage of SCLM-SIM is that the scores in SCLM-SIM are all computed via simulation, which can become computationally  costly. 
Motivated by this, we propose a quicker and more flexible LLM-based approach, where we prompt an LLM to rate the alignment of a candidate reward function with a preference clause. In particular, for each $R\in \mathcal{R}$ and $p_i\in P$, we use a prompt that includes $R$, $p_i$, and the utility feature distributions produced by policy $\Pi(R)$. 
We ask the LLM to rank how well $R$ aligns with the preference clause $p_i$ on a scale from  1 to 5 (see \Cref{app:prompt} for prompt texts). 

\paragraph{Preventing Unintended Utility Shifts and Utility Drop}\label{shifts}
Aligning reward functions to a prioritization prompt may cause (unintended) utility shifts in other features (e.g., due to feature correlations, shifting utility to low-income beneficiaries might shift it away from more educated ones). See \Cref{fig:fair} for a concrete example. 
SCLM offers users the option to explicitly prevent these shifts by adding additional clauses (``objectives'') to the prompt: 
Given a prompt $P$ (e.g., the prompt from \Cref{ex}), for each feature not referenced in the prompt, the user can add a new preference clause requesting a minimum shift in the utility distribution of this feature (e.g., for \Cref{ex} they could add ``do not change the utility distribution for feature $C$''). 
To compute the alignment score $s_i(R)$ between a reward function $R$ and a clause $p_i$=``minimize utility shift for feature $X$'', we compare feature $X$'s utility distribution under the default policy with its utility distribution under the policy $\Pi(R)$. 
Specifically, we quantify the difference using the Earth mover's distance (EMD) between the two distributions.
Afterward, we apply $0$-$1$ normalization to all scores $s_i(R)_{R\in \mathcal{R}}$ for prompt $p_i$, which are input to the social welfare function (along with the alignment scores for the other clauses).

Another potential risk of aligning a reward function with a prioritization prompt is that it can sharply decrease the summed utility generated by all agents: The user might request the prioritization of a subpopulation that does not benefit much from receiving a resource, leading to severe drops in the summed utility generated by all agents. 
Users can address this issue in our model by adding a clause $p_i$=``maximize the total generated utility'' to the prompt.
As the alignment score $s_i(R)$ of $p_i$ with some reward function $R$ we compute the summed utility, i.e., the total number of steps in which arms are in an active state, generated by all agents under the policy $\Pi(R)$ (computed via multiple simulations of the policy on the given instance). We again apply $0$-$1$ normalization to all scores $s_i(R)_{R\in \mathcal{R}}$ for prompt $p_i$.

\subsubsection{Error Bounds for Adjudicator's Selection}\label{sec:error}
\theoremstyle{plain}
Even though we observe in our experiments that the scorer models produce mostly accurate scores, the output scores are oftentimes still a bit noisy. 
To measure how errors propagate through the Social Choice Language Model and how they affect the final reward function selection, we consider the following setup.

Suppose instead of observing the true score vector $\mathbf{s}(R_j) = (s_1(R_j), s_2(R_j), ..., s_l(R_j))$, the Scorer Model (SCLM-SIM or SCLM-LLM) returns a noisy score estimate $\tilde{\mathbf{s}}(R_j)$ with multiplicative noise $\alpha \in (0,1]$  satisfying

\begin{equation}\label{eq:approx}
\alpha \cdot \mathbf{s}(R_j) \leq \tilde{\mathbf{s}}(R_j) \leq \frac{1}{\alpha} \cdot \mathbf{s}(R_j), \quad \forall R_j\in \mathcal{R}.
\end{equation}

\noindent Let $\tilde{R}^* = \arg\max_{R_j \in \mathcal{R}} f_p(\tilde{\mathbf{s}}(R_j))$ be the best reward function under the generalized p-mean function for the observed, noisy scores ($\tilde{R}^*$ will be returned by SCLM); and $R^* = \arg\max_{R_j \in \mathcal{R}} f_p(\mathbf{s}(R_j))$ be the best reward function under the generalized p-mean function for the true, latent score. We define the (relative) regret we encounter by choosing $\tilde{R}^*$ instead of $R^*$ as: 
\begin{equation}
\textit{Relative Regret} = \frac{f_p(\mathbf{s}(R^*)) - f_p(\mathbf{s}(\tilde{R}^*))}{f_p(\mathbf{s}(R^*))}
\end{equation}
The relative regret measures the relative drop in p-mean welfare of the reward function chosen by the adjudicator as compared to the optimal reward function. We show that the relative regret degrades gracefully in the multiplicative error parameter, highlighting that even in the presence of noise, SCLM selects good reward functions with guarantees.

\begin{proposition} The relative regret is bounded by $1-\alpha^2$.
\end{proposition}
\noindent\textit{Proof sketch. } We  observe the monotonicity and positive homogeneity of the generalized p-mean function applied to the input values. These properties allow the application of function $f$ to both sides of Inequality \ref{eq:approx}. Subsequently, using the definition of $R^*$ and $\tilde{R}^*$, we derive the regret bound.
For a complete proof, see Appendix D.

\section{Experiments}

We describe our testset (\Cref{sec:DD}), the compared methods (\Cref{models}), and our experimental results both for dealing with composite prioritization prompts (\Cref{TradeoffExp}) and additionally minimizing unintended side effects (\Cref{sec:fairBias}). 
Following the work of \citet{DBLP:journals/corr/abs-2402-14807}, which constitutes our most important baseline, we use Gemini Pro \cite{DBLP:journals/corr/abs-2312-11805} as the LLM in our experiments.

\subsection{Dataset Description} \label{sec:DD}
\citet{armman_ngo} is a non-profit in India that operates large-scale Maternal and Child Care Mobile Health programs for underserved communities. 
One of their programs disseminates critical health information via weekly automated voice messages. 
The goal of the NGO is to maximize beneficiaries' engagement, i.e., the number of messages they listen to.
A limited number of beneficiaries are called by health workers every week to boost engagement. 
The problem of planning which beneficiaries to call has been modeled and solved as an RMAB, where the good/bad state corresponds to a high/low weekly engagement of the beneficiary. 
We use anonymized data from a quality improvement study conducted in January 2022 \cite{DBLP:journals/aim/VermaSMVGMHTJTT23}. 
For each beneficiary, we have access to their income, education, and age level, which we use as our three features. 
Beneficiaries' historic listenership values are used to estimate their transition probabilities under the passive action \cite{mate2022field}. 
One problem in estimating transition probabilities under the active action is that due to the limited number of service calls made, such transitions are rare. Thus, active transition probability estimates are noisy. To alleviate this issue,
we use the features and passive transition probabilities from ARMMAN together with synthetically generated active transition probabilities. 
Finally, we create three datasets, each consisting of
five sampled RMAB instances with
$N=2100$ arms, a budget of $B=210$ and a time horizon of $T=12$. The three datasets differ in 
how much each feature impacts the effect of applying an active action.  In addition to the real-world domain, we also create three completely synthetic domain datasets (see Appendix B.3 and B.4 for more details on dataset generation).

\paragraph{Problem Instances}
Instances of our problem consist of two parts: A preference prompt and an RMAB instance. 
We initially focus on prioritization prompts. 
Specifically, for each feature $X$, we consider two different prioritization clauses ``Prioritize agents with low/high value of feature $X$''. 
This gives rise to $6$ singular prompts consisting of one prioritization clause, two for each feature.
For composite prompts, we take all combinations of two features and the two prioritization clauses for each feature (e.g. ``Prioritize agents with high value of feature $A$ and also prioritize agents with low value of feature $B$''). 
This results in $3\cdot 4=12$ composite prompts. 
For each domain, we run each prompt on the $15$ RMAB instances from the three datasets.




\subsection{Models \& Baselines}\label{models}
We analyze six different variants of SCLM differing in the used social welfare function (Utilitarian, Egalitarian, Nash) and scorer model (Simulator or LLM), e.g., we denote as \emph{SCLM-SIM-Egalitarian} SCLM with the Simulator Scorer Model and the Egalitarian social welfare function.
In our generator, we generate $4$ candidate reward functions in each step and run $5$ iterations to generate a total of $20$ candidate reward functions.
In addition, we consider several LLM-focused baselines (see \Cref{app:prompt} for~detailed~descriptions): 
\begin{description}
    \item[LLM-Zeroshot] This baseline only queries the LLM once. It asks to return a reward function aligned with the given preference prompt and provides the problem and feature  description as additional context in the prompt.
    \item[DLM] This baseline implements the Decision-Language Model by \citet{DBLP:journals/corr/abs-2402-14807} (see \Cref{app:DLM}). 
    \item[DLM-PromptEngg] This is a modified version of DLM where within the reflection prompt, we include examples for singular queries of how the LLM should reason over the different reward function choices (see \Cref{app:prompt}). 
\end{description}






\subsection{Results: Composite Prioritization Prompts}\label{TradeoffExp}
We analyze the performance on the $12$ composite prompts described above which request the prioritization of two subpopulations (see \Cref{app:comp_Promp} for additional results).

\paragraph{Evaluation Metrics}
As our goal is to fulfill the preferences specified by the user (in contrast to the classic goal of maximizing total utility), we need to quantify the alignment of the returned reward function with the given prompt $P$ to evaluate our models. 
Due to the composite, multi-objective nature of our prompts, we start by measuring the alignment of the returned reward function $R$ with each prioritization clause $p_i\in P$ in a separate evaluation score $e_i(R)$.
For this, we need to quantify how well a given reward function prioritizes the subpopulation specified in the prompt. 
However, as our prompts are written in human language, these subpopulations are not precisely defined (as the prompts only speak of agents with ``high''/``low''  value of some feature $X$).
Notably, one could think that the scores $s_i(R)$ computed in our adjudicator could be used as our evaluation scores $e_i(R)$, as they measure how well a reward $R$ aligns with a prioritization clause $p_i$. 
However, this would create an unfair competitive advantage for  the SCLM compared to our baselines who do not have access to these~scores.

Instead, we assume that the terms ``low'' and ``high'' in the input prompts refer to the most extreme feature values.
Let $p_i$ be some prompt prioritizing agents with a high/low value of some feature $X$. 
As the evaluation score $e_i(R)$, we compute the summed utility generated by the agents with highest/lowest value of $X$ under the policy $\Pi(R)$ normalized by the utility generated by these agents under the default policy $\Pi(R^*)$.\footnote{In  \Cref{app:comp_Promp}, we also check how the results change if we instead interpret ``low''/``high'' to refer to the lowest/highest two or three values. We observe very similar trends.}  
Reflecting the multi-objective nature of our problem, we consider two metrics for measuring the alignment of a reward $R$ with a full composite prompt: sum and minimum of $\%$ change of the utility generated by the two prioritized groups under policy $\Pi(R)$ compared to the default policy, i.e., the sum (resp.\ minimum) of the evaluation scores for $R$.

\begin{figure*} 
  \centering
  \begin{subfigure}{\textwidth}
  \centering
    \includegraphics[trim=0 160 0 0, clip,width=0.85\textwidth]{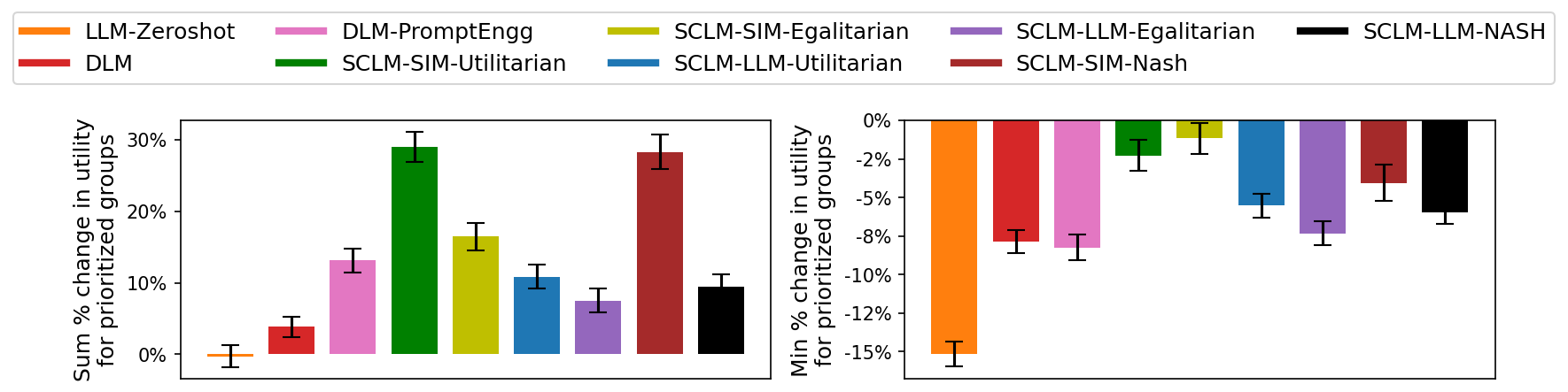}

  \end{subfigure}
  \hfill
  \begin{subfigure}{0.9\textwidth}
  \centering
    \includegraphics[width=0.7\textwidth,height=65pt]{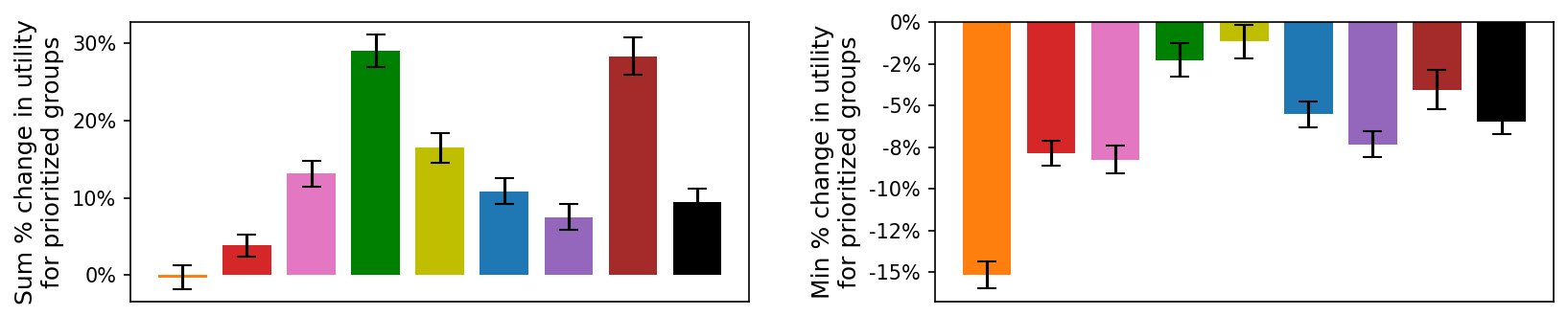}
    \caption{Synthetic domain: sum \% change (left) and minimum of \% change (right) in utility for the two groups prioritized.}
    \label{fig:synth_barchart}
  \end{subfigure}
  \hfill
  \begin{subfigure}{0.9\textwidth}
  \centering
    \includegraphics[width=0.7\textwidth,height=65pt]{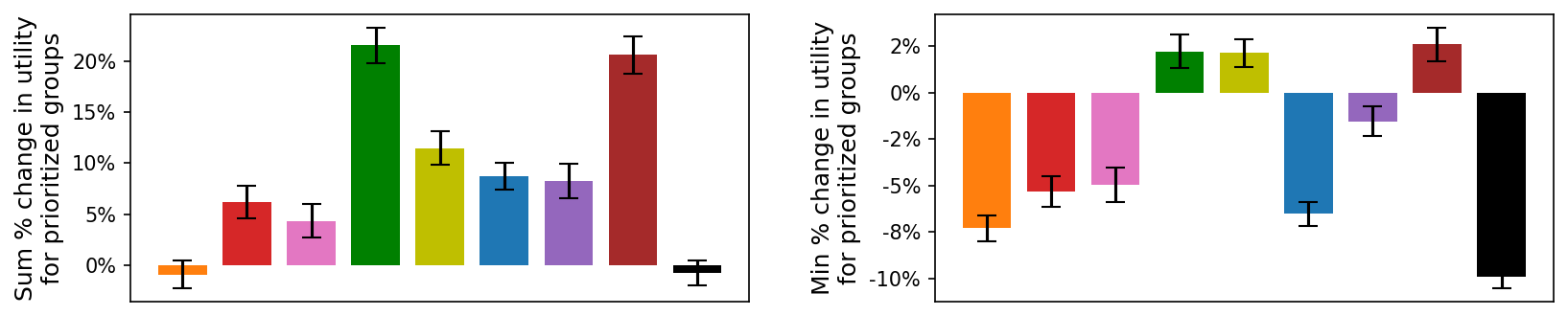}
    \caption{Real-world domain: sum \% change (left) and minimum of \% change in utility for the two groups prioritized.}
    \label{fig:rw_barchart}
  \end{subfigure}
  \hfill
  \caption{ \color{black}Results comparing the quality of reward design methods for composite prioritization prompts.  
  Results are averaged across $180=12\cdot 15$ values: $12$ composite prompts on $15$ RMAB instances (from $3$ datasets). Error bars represent std-error. \color{black}
  }
  \label{fig:tradeoff} \vspace{-0.5cm}
\end{figure*}

\paragraph{Results}
In \Cref{fig:tradeoff}, we show the averaged results from the synthetic and real-world domain.
We depict the average summed and minimum  alignment with the two clauses of the composite prompt, i.e., the minimum/summed change in the utility generated by the prioritized group of agents. 

We start by focusing on SCLM with Simulator Scorer \textit{SCLM-SIM} (green-shaded bars), our strongest method. 
\color{black} On both domains, \textit{SCLM-SIM} significantly outperforms all baselines for both minimum and summed $\%$ change independent of whether the Utilitarian, Egalitarian or Nash social welfare function is chosen. \color{black} 
\textit{SCLM-SIM-Utilitarian} outperforming the baselines for the minimum change and \textit{SCLM-SIM-Egalitarian} outperforming them for the summed change highlights the advantages of the SCLM, as these objectives are not explicitly optimized by the respective models, e.g., 
\textit{SCLM-SIM-Utilitarian} aims at maximizing the summed change and not the minimum one, but still performs well regarding the minimum change.
This indicates that SCLM independent of the chosen welfare function does a better job at picking effective and aligned reward functions (on the Pareto front).
Comparing \textit{SCLM-SIM-Utilitarian} and \textit{SCLM-SIM-Egalitarian}, the two methods exhibit a big difference under the summed change criterion, while the difference regarding the minimum change is much smaller. 
Examining individual instances we find in \Cref{app:comp_Promp} that both functions lead to very different selections on the instance level; unsurprisingly, the Egalitarian method creates rewards that benefit both groups in a more balanced fashion. \color{black} For the Nash welfare function, we found that the performance was similar, yet slightly inferior to the Utilitarian welfare function in all relevant evaluation dimensions. \color{black}


If we replace Simulation Scorer with LLM Scorer, performance of SCLM decreases, but is better than all of our three baselines. 
The difference between LLM and Simulation Scorer highlights the advantage of the additional information acquired through more complex and computationally expensive simulation method. 
Regarding the performance of baselines, 
our DLM baseline with prompt engineering \textit{DLM-PromptEngg} improves slightly upon the results of DLM in the synthetic domain, while in real-world domain, their performance is similar. 
This suggests that prompt engineering itself is not sufficient to adequately deal with the multi-objective nature of composite prompts; an external component (like our adjudicator) is needed. Finally, LLM-zeroshot consistently performs the worst, which highlights the non-triviality of  reward design problem and the need for a guided extensive search within reward function search space.



\color{black}

\begin{table}[ht]
\caption{\textcolor{black}{Results comparing different reward function selection strategies, aggregated across three real-world datasets. Higher summed \% change in desired feature(s) implies better alignment with prioritization clauses, whereas less unintended shift and less \% drop in utility are better.}}
\vspace{5pt}
\centering 
    \centering
    \renewcommand{\arraystretch}{1.2}
    \resizebox{1\columnwidth}{!}{
    \begin{tabular}{|c|c|c|}
\hline
\textbf{Method}         & \textbf{\begin{tabular}[c]{@{}c@{}}Summed \% Change in\\ Desired Feature(s)\end{tabular}} & \textbf{\begin{tabular}[c]{@{}c@{}}Unintended\\ Shift\end{tabular}} \\ \hline
DLM-PrioritizationOnly  & 6.809$\pm$0.86                                                                            & 0.302$\pm$0.02                                                      \\ \hline
DLM-ExtendedPrompt-Fair & -0.254$\pm$0.73                                                                           & 0.276$\pm$0.02                                                      \\ \hline
SCLM-PrioritizationOnly                  & 13.131$\pm$0.86                                                                           & 0.316$\pm$0.02                                                      \\ \hline
SCLM-ExtendedPrompt-Fair  & 15.364$\pm$0.94                                                                           & 0.099$\pm$0.01                                                      \\ \hline
\end{tabular}}
\label{table1}
\end{table}
\color{black}
\subsection{Addressing Fairness and Biases}\label{sec:fairBias}

As discussed in \Cref{ad}, we can also use our pipeline to prevent unintended side-effects of aligning reward functions with prioritization clauses, i.e., 
\begin{enumerate*}[label=(\roman*)]
\item shifts in the utility feature distribution of features not included in the prompt,  
\item drops in the total generated utility, and 
\item arbitrary weighted combinations of the above two goals and additional preference clauses.
\end{enumerate*}
We focus on (i) here and relegate the results for (ii) and (iii), which paint a very similar picture, to Appendix D.3.

We analyze all $6$ singular and $12$ composite prioritization prompts (see \Cref{sec:DD}), where we add additional clauses to prevent shifts in the utility distribution of all features not referenced in the prompt. We use the simulator scorer with a Utilitarian social welfare function and call the resulting model SCLM-ExtendedPrompt-Fair.
As baselines, we consider DLM only prompted with the prioritization clause(s) (called DLM-PrioritizationOnly) and DLM prompted with the prioritization clause(s) and clause(s) together for a request for minimizing of utility shifts for the other features (called DLM-ExtendedPrompt-Fair). We also consider SCLM-SIM-Utilaterian only prompted with the prioritization clause(s) (called SCLM-PrioritizationOnly).
See Appendix E for more details on the prompts.

To compute the alignment with prioritization clauses, similar to \Cref{TradeoffExp}, we compute average change in utility generated by prioritized subpopulations.
To quantify unintended utility shifts, we compute the average Earth mover's distance between utility feature distribution under the candidate and default reward function for each feature not included in one of the prioritization clauses.

\Cref{table1} shows the results. 
Comparing \emph{DLM-PrioritizationOnly} and \emph{DLM-ExtendedPrompt-Fair}, we find that adding additional objective to the prompt does not result in a better performance for real-world domains. 
In contrast, \emph{SCLM-ExtendedPrompt-Fair} which incorporates unintended shifts in the selection chooses reward functions resulting in significantly higher utility increases for prioritized subpopulations and significantly fewer unintended utility shifts. 
The fact that SCLM performs advantageously for both (conflicting) objectives highlights the quality of the pipeline and its capabilities to effectively address multiple objectives (of different types). We see similar results in synthetic domain (see Table 5 in the appendix). 
\section{Discussion}
We present a customizable Social Choice Language Model to handle the multi-objective nature of preference prompts in reward design for RMABs. 
We showcase how methods from social choice can be used to improve the quality and transparency of decision-making of LLM-based frameworks, as we present an adjudicator component that makes the final decision from options generated by the LLM.   
SCLM significantly improves the quality of the chosen reward functions.
We demonstrate that SCLM can not only handle composite prioritization prompts but arbitrary prompts containing multiple objectives, e.g.,  balancing the prioritization of subpopulations with the total utility generated by all agents. 
For future work, SCLM can be applied to other problems from multiagent planning and reinforcement learning. Further, SCLM can easily be extended to handle multiple preference prompts specified by different users.




\section{Acknowledgments}
This material is based upon work supported by the AI Research Institutes Program funded by the National Science Foundation under the AI Institute for Societal Decision Making (NSF AI-SDM), Award No. 2229881.
\bibliography{sample}

\newpage

\clearpage
\appendix

\section*{Ethics Statement}
\subsection*{ARMMAN domain}
\paragraph{Data Usage, Collection and Consent} The data used for the realworld domain contains fully-anonymyzed datasets. Our experiments constitute secondary analysis of the data and are approved by ARMMAN's ethics board. The paper does not involve any realworld deployment of the proposed algorithms for ARMMAN. For data collection consent is taken from each beneficiary in ARMMAN's automated voice call program.  Data exchange and use was regulated through clearly defined exchange protocols
including anonymization, read-access only to researchers, restricted use of the data for research purposes only, and approval by
ARMMAN’s ethics review committee.
\paragraph{Accessibility of Health Information} In our simulation experiments, we change the reward functions for every agent in mobile health program. But this only improves the quality of service calls and no health information is withheld from any agent. Participants in the program can still request service calls through free missed call service.

\newpage
\clearpage
\section*{Appendix}

\section{Additional Related Work}\label{app:RW}

\section{Additional Material for \Cref{sec:Prob}}\label{app:Prob}

\begin{figure}
        \centering
        \begin{subfigure}{0.23\textwidth}\includegraphics[width=\textwidth]{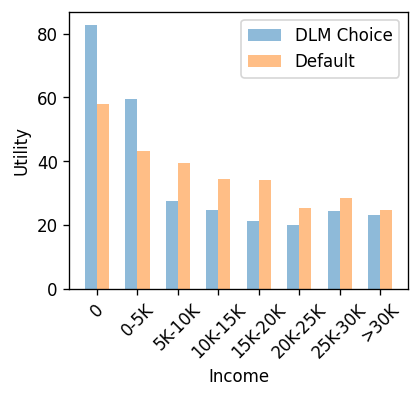}
\caption{Income feature.}
\label{fig:fair1}
        \end{subfigure}
        \begin{subfigure}{0.23\textwidth}\includegraphics[width=\textwidth]{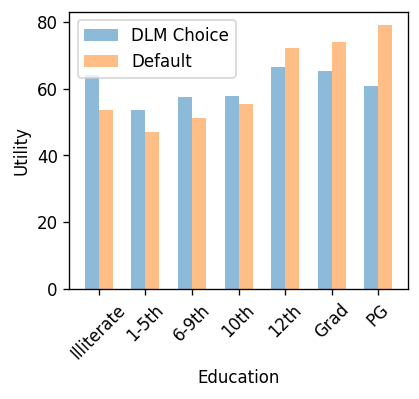}\caption{Education feature.}\label{fig:fair3}
        \end{subfigure}
        \caption{Utility feature distributions for the default reward function (orange) and reward function returned for prompt ``Prioritize  agents with low income'' (blue). The $x$-axis depicts the values of the feature and the $y$-axis the total utility generated by all agents with this feature value.}
        \label{fig:fair}
    \end{figure}
    
\paragraph{Unintended Utility Shifts}
Moving from the default reward function to a reward function aligned with a given (prioritization) prompt causes shifts within the distribution of resources and utility.  
Due to correlations between features, 
this might change might lead to unintended utility shifts for features not specified in the prompt. In \Cref{fig:fair}, we show an instance of our experiments. 
We present the utility feature distribution for the two features income and education for two reward functions: 
The reward function selected by DLM for the prompt `` Prioritize agents with low income ''  (orange) and the default reward (blue). 
While the utility generated by low-income agents increases, the utility generated by highly educated agents decreases, a side-effect a user is likely unaware of and that might conflict with their allocation goals.

\section{Implementation Details}\label{app:implementation}

\subsection{Whittle Index Policy}\label{app:WI}
To formulate the Whittle Index Policy, first we define the value function for an arm $i\in[N]$ under the subsidy $\lambda$ as
\begin{align}
\label{eq:value_function_each_arm}
V_i^{\lambda}(s) = \max_{a\in\{0,1\}} Q_i(s,a_i,\lambda).
\end{align}
Here, $Q_i(s,a_i,\lambda)$ measures the expected discounted cumulative future reward where a subsidy $\lambda$ is added to the reward when the passive action is taken. The Whittle index associated to the state $s_i$ is then defined as:
\begin{align}
W_i(s_i):=\inf_{m}\left\{Q_i(s_i,a_i=0,m) = Q_i(s_i,a_i=1,m)\right\}.
\end{align}
The whittle index is thus the value of subsidy such that the passive ($a=0$) and active ($a=1$) actions have the same value. Intuitively, this captures the value of taking active action on an arm.

To implement the Whittle Index computation, we use the method in \cite{qian2016restless} based on binary search. Additionally, for all the experiments, we cache computed whittle index for a given set of transition probabilities and reward functions to reduce computation time.

\subsection{DLM Pipeline}\label{app:DLM}
In our work, we use a modified version of the DLM pipeline~\cite{DBLP:journals/corr/abs-2402-14807}, which employs the Whittle Index policy as the planner for RMAB. Our approach differs from \citet{DBLP:journals/corr/abs-2402-14807}  in that we use the Whittle Index policy specifically for simulating RMAB, whereas \citet{DBLP:journals/corr/abs-2402-14807}  use the PreFeRMAB policy (Zhao and Behari et al. 2023). This modification allows for faster and more stable simulations and effectively decouples the learning problem from the planning problem.

Specifically, the DLM consists of three components. First, a user provides a natural language preference $P$. We then create a prompt for LLM which includes the context of the RMAB problem, the preference $P$, a description of features available, and the index of those features in the feature array. Finally, the LLM is asked to generate the Python code of the reward function in text format. We describe the prompt used in Figures \ref{app:prompt_gen_synth} and \ref{app:prompt_gen_rw}.

The LLM is queried $n_p$ times to obtain $n_p$ reward functions. For each reward function, we also compute the reward distribution over all the features. Next, given all the generated reward functions and their corresponding reward distributions, we query the LLM to select the best reward function. We describe the prompt used in Figures \ref{app:prompt_choose_synth} and \ref{app:prompt_choose_rw}. This is called the reflection step. The best reward function is then used inside the prompt for next step of generating $n_p$ reward functions. This process is repeated $n_r$ times to obtain $n_p \cdot n_r$ reward functions. In all our reward function generation experiments, we query the LLM $n_p=4$ times and run the reflection loop $n_r=5$ times resulting in 20 candidate reward functions.

As an LLM, we use the Gemini Pro model by Google. We query the LLM using python based API from the \textit{generative-ai-python} library. 

\subsection{Synthetic Dataset Generation}\label{app:dataset_synth}
An RMAB problem is defined by the transition probabilities of Markov Decision Process governing every arm and the reward functions. We consider a 2-state, 2-action Markov decision process in all our experiments. This results in 8 transition probability parameters $P_i(s, a, s')\; \forall s\in \{0,1\}, a\in \{0,1\}, s' \in \{0,1\}$ for every arm $i$.
Out of these 8 parameters, we only need to define 4 parameters $P_i(s, a, s'=1)\; \forall s\in \{0,1\}, a\in \{0,1\}$ independently, and the rest 4 parameters can be calculated as the compliment values $P_i(s, a, s'=0) = 1-P_i(s, a, s'=1)\; \forall s\in \{0,1\}, a\in \{0,1\} $.

To simulate the effects of conflicting preferences and trade-offs in a controlled setup, we configure the four transition probability parameters to depend on the features describing each arm. We consider each arm to be characterized by a vector of continuous features $f$, with three values ranging between 0 and 1 ($f \in [0,1]^3$).  We use the following setup to stochastically generate the 4 transition probability parameters.
\begin{enumerate}
    \item For every arm $i$, uniformly sample the two passive probability parameters $P_i(s, a=0, s'=1) \sim \text{Uniform}(0, 1),\; \forall s\in \{0,1\}$.
    \item For every arm $i$, uniformly sample the three features $f \sim Uniform([0,1]^3)$
    \item Define a three-dimensional weight vector $W\in [0,1]^3$ and a standard deviation parameter $\sigma$.
    \item Sample the effect of intervention from the normal distribution as $\delta \sim \mathcal{N}(\Delta, \sigma),$
where $\Delta = W\cdot f^T$.
    \item Calculate active transition probabilities as $P(s, a=1, s') = P(s,a=0, s')+\delta$
    \item Calculate the complimentary probabilities as $P_i(s, a, s'=0) = 1-P_i(s, a, s'=1),\; \forall s\in \{0,1\}, a\in \{0,1\} $
\end{enumerate}

The magnitude of the weight value determines the extent to which a feature influences the effect of the intervention, while the direction indicates whether a low or high feature value amplifies this effect. The standard deviation parameter, $\sigma$, controls the spread of sampled probabilities around the mean effect of the intervention. Table \ref{tab:dataset} shows the Weight vector used for the three synthetic datasets generated in our experiments. We consider weight values that describe the following scenarios : i) all weights are roughly equal (dataset 1) ii) one of the features has much higher weight than the other two (dataset 2) and we switch which feature has maximum weight (dataset 3).
\renewcommand{\arraystretch}{1.5}
\begin{table}[]
\centering
\caption{Weight vector parameters for different synthetic datasets.}
\resizebox{0.8\columnwidth}{!}{
\begin{tabular}{|c|c|c|c|l|}
\hline
                   & \textbf{Feature A} & \textbf{Feature B} & \textbf{Feature C} & $\sigma$ \\ \hline
\textbf{Dataset 1} & 0.8                & -1.5               & 1                  & 0.1      \\ \hline
\textbf{Dataset 2} & 10                 & -1.5               & 1                  & 0.1      \\ \hline
\textbf{Dataset 3} & 1                  & -1.5               & 10                 & 0.1      \\ \hline
\end{tabular}
}
\label{tab:dataset}
\end{table}

\subsection{Real World Dataset}\label{app:dataset_rw}
As described in Section 6.1, one challenge in estimating transition probabilities under active actions is the rarity of such transitions due to the limited number of service calls. Consequently, the estimates for active transition probabilities tend to be noisy. To mitigate this issue, we follow a procedure similar to that in the synthetic domain to construct the datasets, but with some modifications. Specifically, instead of uniformly sampling passive probabilities in step 1, we use the transition probabilities estimated from real-world data \cite{DBLP:journals/aim/VermaSMVGMHTJTT23}. Next, in step 2, we use the attributes describing the real-world beneficiaries to define features. Thus, instead of features A, B and C, we use features Age, Income and Education. Finally, we use the same weight vector values as in Table \ref{tab:dataset} to generate three real world datasets. This allows us to generate multiple datasets with varying levels of effect of intervention based on features while still using the realistic passive transition probabilities and feature distributions.

\subsection{Hyperparameters}\label{app:hyper}
We run all experiments with $N=2100$ arms, $B=210$ as budget, $T=12$ weeks and $\gamma=0.9$ as discount factor. For every dataset, 5 RMAB instances are generated based on the different weight vectors as described in the section above. 
Additionally, we run each RMAB simulation with 10 different seeds. We estimate the cumulative rewards and feature-level utility by calculating the mean of the discounted sum of rewards over $T$ timesteps across all 10 seeds.

\subsection{Computing Resources}\label{app:computing}
All experiments are run on a machine with 16GB RAM and M1 chip CPU. For every RMAB instance, it took roughly 3 hours to generate 360 candidate reward functions (20 reward functions for each of the 18 prompts). The primary bottleneck in speed was the API call limits for using LLM (Gemini Pro).


\section{Error Bound in Adjudicator}
\theoremstyle{plain}


\paragraph{Multiplicative error}
Suppose instead of observing the true score $s_i(R_j) $, the Scorer Model (SCLM-SIM or SCLM-LLM) returns a noisy score estimate $\tilde{s}_i(R_j)$ under  multiplicative noise $\alpha_i^j \in (0,1]$  given as
\begin{equation}
\alpha_i^j \cdot s_i(R_j) \leq \tilde{s}_i(R_j) \leq \frac{1}{\alpha_i^j} \cdot s_i(R_j), \quad \forall i \in [l], R_j\in \mathcal{R}.
\end{equation}

Let $\mathbf{s}(R_j) = (s_1(R_j), s_2(R_j), ..., s_l(R_j))$ represent the true score vector (and similarly, $\tilde{\mathbf{s}}(R_j)$ the approximate score vector). 
If we consider $\alpha$ as the worst approximation factor across all scores, i.e., $\alpha = \min_{i \in [l], j \in [|\mathcal{R}|]} \alpha_i^j$, we can then write the vector form of the approximation as
\begin{equation}
\alpha \cdot \mathbf{s}(R_j) \leq \tilde{\mathbf{s}}(R_j) \leq \frac{1}{\alpha} \cdot \mathbf{s}(R_j), \quad \forall R_j\in \mathcal{R}.
\end{equation}
This is true because $\alpha \cdot s_i(R_j) \le \alpha_i^j \cdot s_i(R_j) $ and $1/\alpha_i^j \cdot s_i(R_j) \le 1/\alpha \cdot s_i(R_j), \; \forall i \in [l], R_j \in \mathcal{R}$.

\paragraph{Properties of Generalized p-mean}
\begin{lemma}\label{lem: social-welfare-monotonicity}
    For all strictly positive vectors $\mathbf{x} \in \mathbb{R}^N_{> 0}$, the generalized $p$-mean is:
    \begin{itemize}
        \item \textbf{positively homogeneous:} $f_p(\lambda \mathbf{x}) = \lambda f_p(\mathbf{x})$ for all $\lambda > 0$, and
        \item \textbf{monotonic:} if $\mathbf{x} \le \mathbf{y}$ element wise, then $f_p(\mathbf{x}) \le f_p(\mathbf{y})$.
    \end{itemize}
    These properties hold for all $p \in [-\infty, 1]$.
\end{lemma}

\begin{proof}
\textbf{Positive homogeneity:}

\textbf{Case 1: } $p = -\infty$

In this case:
\[
f_{-\infty}(\mathbf{x}) = \min_{i \in [N]} x_i
\]
Let $x_j = \min_{i \in [N]} x_i$. Then, for $\lambda > 0$, it must be true that $\lambda x_j = \min_{i \in [N]} \lambda x_i$. Since $f_{-\infty}(\lambda \mathbf{x}) = \lambda f_{-\infty}(\mathbf{x})$, $f_{-\infty}(\mathbf{x})$ is positive homogeneous.

\textbf{Case 2: } $p \not\in \{-\infty, 0\}$

Let $\lambda > 0$ and $\mathbf{x} = (x_1, \dots, x_N) \in \mathbb{R}^N_{>0}$. Then
\begin{align*}
f_p(\lambda \mathbf{x}) 
= \left( \frac{1}{N} \sum_{i=1}^N (\lambda x_i)^p \right)^{1/p} 
&= \left( \frac{1}{N} \sum_{i=1}^N \lambda^p x_i^p \right)^{1/p} \\
&= \left( \lambda^p \cdot \frac{1}{N} \sum_{i=1}^N x_i^p \right)^{1/p}\\
&= \lambda \left( \frac{1}{N} \sum_{i=1}^N x_i^p \right)^{1/p}
= \lambda f_p(\mathbf{x})
\end{align*}
so $f_p$ is positively homogeneous.

\textbf{Case 3: } $p = 0$

We use the definition of the geometric mean:
\begin{align*}
f_0(\lambda \mathbf{x}) 
= \left( \prod_{i=1}^N \lambda x_i \right)^{1/N} 
&= \left( \lambda^N \prod_{i=1}^N x_i \right)^{1/N} \\
&= \lambda \left( \prod_{i=1}^N x_i \right)^{1/N} 
= \lambda f_0(\mathbf{x})
\end{align*}
Thus, positive homogeneity also holds for $p = 0$.

\textbf{Monotonicity:}
Suppose $\mathbf{x}, \mathbf{y} \in \mathbb{R}^N_{>0}$ such that $x_i \le y_i$ for all $i \in [N]$.

\textbf{Case 1: } $p=-\infty$

Let $x_u = \min_{i \in [N]} x_i$ and $y_v = \min_{i \in [N]} y_i$. Because $x_u$ is minimum over $\mathbf{x}$, $x_u \le x_v$. Because $x_i \le y_i$ for all $i \in [N]$, it must be true that $x_v \le y_v $. Hence, $x_u \le y_v$ or $\min_{i \in [N]} x_i \le \min_{i \in [N]} y_i$ and thus $f_{-\infty}(\mathbf{x})$ is monotonic in $\mathbf{x}$.

\textbf{Case 2: }$p \notin \{-\infty, 0\}$

Since the power function $u \mapsto u^p$ is monotonically increasing for all $p > 0$ (and $u > 0$), we have $x_i^p \le y_i^p$ for all $i \in [N]$. Therefore,
\[
\frac{1}{N} \sum_{i=1}^N x_i^p \le \frac{1}{N} \sum_{i=1}^N y_i^p
\]
Now, since $u \mapsto u^{1/p}$ is increasing for $p > 0$, we can apply the power of $1/p$ and the inequality will be preserved because both sides are positive. Hence,
\[
\left(\frac{1}{N} \sum_{i \in [N]} x_i^p\right)^{\!1/p} \le \left(\frac{1}{N} \sum_{i \in [N]} y_i^p\right)^{\!1/p}
\]
or 
\[
f_p(\mathbf{x}) \le f_p(\mathbf{y})
\]
When $p<0$, $u \mapsto u^{p}$ is decreasing function (for u>0). Hence, we have $x_i^p \ge y_i^p$ for all $i \in [N]$.Therefore, 
\[
\frac{1}{N} \sum_{i=1}^N x_i^p \ge \frac{1}{N} \sum_{i=1}^N y_i^p
\]
However,  since $u \mapsto u^{1/p}$ is decreasing for $p < 0$, we can apply the power of $1/p$ and the inequality will be flipped. Thus, we would get the same result as before
\[
\left(\frac{1}{N} \sum_{i \in [N]} x_i^p\right)^{\!1/p} \le \left(\frac{1}{N} \sum_{i \in [N]} y_i^p\right)^{\!1/p}
\]
or 
\[
f_p(\mathbf{x}) \le f_p(\mathbf{y})
\]

\textbf{Case 3: }$p=0$

We use the fact that that logarithm is a strictly increasing function (i.e., $u<v \Leftrightarrow \log u < \log v, \forall u,v>0$), and then reduce the product to sum

\begin{align*}
\log f_0(\mathbf{x}) 
&= \frac{1}{N} \sum_{i=1}^N \log x_i 
\le \frac{1}{N} \sum_{i=1}^N \log y_i 
= \log f_0(\mathbf{y}) \\
&\Rightarrow \quad f_0(\mathbf{x}) \le f_0(\mathbf{y})
\end{align*}
Hence, the generalized p-mean function is both positively homogeneous and monotonic.
\end{proof}

\paragraph{Proof of Proposition 1:}
\begin{proof}
From our approximation definition in Inequality \ref{eq:approx}, we have
\[
\alpha \cdot \mathbf{s}(R_j) \leq \tilde{\mathbf{s}}(R_j) \leq \frac{1}{\alpha} \cdot \mathbf{s}(R_j), \quad \forall R_j\in \mathcal{R}.
\]
Because $f$ is monotonic (from Lemma \ref{lem: social-welfare-monotonicity}), we can apply $f$ on both sides of the inequality without changing the direction. Thus we get
\begin{equation}\label{eq:f_approx}
f_p(\alpha \cdot \mathbf{s}(R_j)) \leq f_p(\tilde{\mathbf{s}}(R_j)) \leq f_p(\frac{1}{\alpha} \cdot \mathbf{s}(R_j)), \quad \forall R_j\in \mathcal{R}.
\end{equation}
Next, because $f$ has positive homogeneity (from Lemma \ref{lem: social-welfare-monotonicity}) for any scalar \( \alpha > 0 \), we have:

\begin{equation}
f_p(\alpha \cdot \mathbf{s}(R_j)) = \alpha\cdot f_p(\mathbf{s}(R_j)),
\end{equation}
and similarly,

\begin{equation}
f_p\left( \frac{1}{\alpha}\cdot  \mathbf{s}(R_j) \right) = \frac{1}{\alpha} \cdot f_p(\mathbf{s}(R_j)).
\end{equation}
Combining the above three observations, we get
\begin{equation}\label{eq:error}
\alpha \cdot f_p(\mathbf{s}(R_j)) \leq f_p(\tilde{\mathbf{s}}(R_j)) \leq \frac{1}{\alpha} \cdot f_p(\mathbf{s}(R_j)), \forall R_j\in \mathcal{R}.
\end{equation}
Since the Inequality \ref{eq:error} is true for all $R_j \in \mathcal{R}$, we can use $R_j = R^*$ to get:
\begin{equation}
f_p(\tilde{\mathbf{s}}(R^*)) \geq \alpha \cdot f_p(\mathbf{s}(R^*)).
\end{equation}
Since \( \tilde{R}^* \) maximizes \( f_p(\tilde{\mathbf{s}}(R_j)) \), we have:
\begin{equation}
f_p(\tilde{\mathbf{s}}(\tilde{R}^*)) \geq f_p(\tilde{\mathbf{s}}(R^*)).
\end{equation}
Now, using the Inequality \ref{eq:error} for $R_j = \tilde{R}^*$, we get:
\begin{equation}
f_p(\mathbf{s}(\tilde{R}^*)) \geq \alpha \cdot f_p(\tilde{\mathbf{s}}(\tilde{R}^*)).
\end{equation}
Combining the above three results, we get
\begin{equation}
f_p(\mathbf{s}(\tilde{R}^*)) \geq \alpha \cdot f_p(\tilde{\mathbf{s}}(\tilde{R}^*)) \geq \alpha \cdot f_p(\tilde{\mathbf{s}}(R^*)) \geq \alpha^2 \cdot f_p(\mathbf{s}(R^*)).
\end{equation}
Thus:

\begin{align}
\frac{f_p(\mathbf{s}(R^*)) - f_p(\mathbf{s}(\tilde{R}^*))}{f_p(\mathbf{s}(R^*))} 
&\leq \frac{f_p(\mathbf{s}(R^*)) - \alpha^2 \cdot f_p(\mathbf{s}(R^*))}{f_p(\mathbf{s}(R^*))} \\
&= \frac{(1 - \alpha^2) \cdot f_p(\mathbf{s}(R^*))}{f_p(\mathbf{s}(R^*))} 
= 1 - \alpha^2
\end{align}

\begin{equation}
\textit{Relative Regret} =\frac{f_p(\mathbf{s}(R^*)) - f_p(\mathbf{s}(\tilde{R}^*))}{f_p(\mathbf{s}(R^*))} \le 1 - \alpha^2
\end{equation}
\end{proof}

\section{Additional Results}

\subsection{Composite Prioritization Prompts} \label{app:comp_Promp}

In Tables \ref{app:full_table_synth} and \ref{app:full_table_rw}, we show results aggregated for the synthetic and real world datasets. Specifically, we show the Sum of \% change in utility and Minimum of \% change in utility not just for the highest/lowest value of a feature but for top/bottom two and three buckets.  We also include results from the scorer model that optimizes for Nash Social Welfare function.

Overall, we observe that the Simulator-based scorer (SCLM-SIM) performs best in all scenarios, both when optimizing for the Utilitarian objective or the Egalitarian objective.  It is also worth noting that when optimizing for one of the objectives (for instance, Utilitarian objective maximizes the sum of \% change in utility), SCLM outperforms baselines in the objective it is not explicitly optimizing for (for instance, Minimum of \% change in utility).
Lastly, we observe that optimizing for the Nash objective yields very similar performance as optimizing for the Utilitarian objective.

\begin{table*}[]
\renewcommand{\arraystretch}{1.5}
\caption{Results summary comparing different reward function choice strategies aggregated across the three synthetic Datasets. Cells in bold indicate the top 2 best values (higher is better). DLM: Decision-Language Model, SCLM: Social Choice Language Model, SCLM-SIM: Simulation based Scorer Model, SCLM-LLM: LLM based Scorer Model}
\resizebox{\textwidth}{!}{
\begin{tabular}{|c|c|ccc|ccc|}
\hline
                & \textbf{Social Welfare Function} & \multicolumn{3}{c|}{\textbf{Minimum of \% change in utility}}                                                                              & \multicolumn{3}{c|}{\textbf{Minimum of \% change in utility}}                                                                              \\ \hline
\textbf{Method} &                                  & \multicolumn{1}{c|}{\textbf{top/bottom 1 bucket}}   & \multicolumn{1}{c|}{\textbf{top/bottom 2 buckets}}  & top/bottom 3 buckets           & \multicolumn{1}{c|}{\textbf{top/bottom 1 bucket}}   & \multicolumn{1}{c|}{\textbf{top/bottom 2 buckets}}  & \textbf{top/bottom 3 buckets}  \\ \hline
LLM-zeroshot    & \multirow{3}{*}{}                & \multicolumn{1}{c|}{-0.266$\%\pm$1.51$\%$}          & \multicolumn{1}{c|}{-0.228$\%\pm$1.3$\%$}           & -3.269$\%\pm$1.14$\%$          & \multicolumn{1}{c|}{-15.16$\%\pm$0.81$\%$}          & \multicolumn{1}{c|}{-13.249$\%\pm$0.71$\%$}         & -12.485$\%\pm$0.61$\%$         \\ \cline{1-1} \cline{3-8} 
DLM             &                                  & \multicolumn{1}{c|}{3.879$\%\pm$1.42$\%$}           & \multicolumn{1}{c|}{5.379$\%\pm$1.27$\%$}           & 1.89$\%\pm$0.99$\%$            & \multicolumn{1}{c|}{-7.844$\%\pm$0.74$\%$}          & \multicolumn{1}{c|}{-6.249$\%\pm$0.62$\%$}          & -6.761$\%\pm$0.49$\%$          \\ \cline{1-1} \cline{3-8} 
DLM-PromptEngg  &                                  & \multicolumn{1}{c|}{7.607$\%\pm$1.53$\%$}           & \multicolumn{1}{c|}{8.665$\%\pm$1.44$\%$}           & 4.087$\%\pm$1.24$\%$           & \multicolumn{1}{c|}{-9.301$\%\pm$0.82$\%$}          & \multicolumn{1}{c|}{-7.17$\%\pm$0.81$\%$}           & -7.767$\%\pm$0.71$\%$          \\ \hline
SCLM-SIM        & \multirow{2}{*}{Utilitarian}     & \multicolumn{1}{c|}{\textbf{28.944$\%\pm$2.12$\%$}} & \multicolumn{1}{c|}{\textbf{21.936$\%\pm$1.55$\%$}} & \textbf{13.654$\%\pm$1.07$\%$} & \multicolumn{1}{c|}{\textbf{-2.278$\%\pm$1.02$\%$}} & \multicolumn{1}{c|}{\textbf{-3.579$\%\pm$0.83$\%$}} & \textbf{-3.559$\%\pm$0.55$\%$} \\ \cline{1-1} \cline{3-8} 
SCLM-LLM        &                                  & \multicolumn{1}{c|}{14.348$\%\pm$1.66$\%$}          & \multicolumn{1}{c|}{10.304$\%\pm$1.22$\%$}          & 6.333$\%\pm$0.99$\%$           & \multicolumn{1}{c|}{-3.973$\%\pm$0.76$\%$}          & \multicolumn{1}{c|}{-4.428$\%\pm$0.6$\%$}           & -4.817$\%\pm$0.54$\%$          \\ \hline
SCLM-SIM        & \multirow{2}{*}{Egalitarian}     & \multicolumn{1}{c|}{16.448$\%\pm$1.95$\%$}          & \multicolumn{1}{c|}{11.425$\%\pm$1.34$\%$}          & 8.6$\%\pm$0.94$\%$             & \multicolumn{1}{c|}{\textbf{-1.176$\%\pm$1.01$\%$}} & \multicolumn{1}{c|}{\textbf{-2.028$\%\pm$0.71$\%$}} & \textbf{-1.833$\%\pm$0.49$\%$} \\ \cline{1-1} \cline{3-8} 
SCLM-LLM        &                                  & \multicolumn{1}{c|}{11.421$\%\pm$1.62$\%$}          & \multicolumn{1}{c|}{7.845$\%\pm$1.21$\%$}           & 4.141$\%\pm$0.92$\%$           & \multicolumn{1}{c|}{-4.877$\%\pm$0.73$\%$}          & \multicolumn{1}{c|}{-4.373$\%\pm$0.59$\%$}          & -4.68$\%\pm$0.5$\%$            \\ \hline
SCLM-SIM        & \multirow{2}{*}{Nash}            & \multicolumn{1}{c|}{\textbf{28.262$\%\pm$2.42$\%$}} & \multicolumn{1}{c|}{\textbf{20.416$\%\pm$1.73$\%$}} & \textbf{11.102$\%\pm$1.19$\%$} & \multicolumn{1}{c|}{-4.053$\%\pm$1.18$\%$}          & \multicolumn{1}{c|}{-5.408$\%\pm$0.89$\%$}          & -5.261$\%\pm$0.62$\%$          \\ \cline{1-1} \cline{3-8} 
SCLM-LLM        &                                  & \multicolumn{1}{c|}{9.478$\%\pm$1.68$\%$}           & \multicolumn{1}{c|}{6.608$\%\pm$1.23$\%$}           & 2.957$\%\pm$1.02$\%$           & \multicolumn{1}{c|}{-5.973$\%\pm$0.77$\%$}          & \multicolumn{1}{c|}{-5.782$\%\pm$0.6$\%$}           & -6.117$\%\pm$0.54$\%$          \\ \hline
\end{tabular}
\label{app:full_table_synth}
}
\end{table*}

\begin{table*}[]
\renewcommand{\arraystretch}{1.5}
\caption{Results summary comparing different reward function choice strategies aggregated across the three Real World Datasets. Cells in bold indicate the top 2 best values (higher is better). DLM: Decision-Language Model, SCLM: Social Choice Language Model, SCLM-SIM: Simulation based Scorer Model, SCLM-LLM: LLM based Scorer Model}
\resizebox{\textwidth}{!}{
\begin{tabular}{|c|c|ccc|ccc|}
\hline
                & \textbf{Social Welfare Function} & \multicolumn{3}{c|}{\textbf{Minimum of \% change in utility}}                                                                              & \multicolumn{3}{c|}{\textbf{Minimum of \% change in utility}}                                                                             \\ \hline
\textbf{Method} &                                  & \multicolumn{1}{c|}{\textbf{top/bottom 1 bucket}}   & \multicolumn{1}{c|}{\textbf{top/bottom 2 buckets}}  & top/bottom 3 buckets           & \multicolumn{1}{c|}{\textbf{top/bottom 1 bucket}}  & \multicolumn{1}{c|}{\textbf{top/bottom 2 buckets}}  & \textbf{top/bottom 3 buckets}  \\ \hline
LLM-zeroshot    & \multirow{3}{*}{}                & \multicolumn{1}{c|}{-0.893$\%\pm$1.38$\%$}          & \multicolumn{1}{c|}{-5.541$\%\pm$0.79$\%$}          & -10.436$\%\pm$0.46$\%$         & \multicolumn{1}{c|}{-7.285$\%\pm$0.69$\%$}         & \multicolumn{1}{c|}{-7.003$\%\pm$0.45$\%$}          & -6.971$\%\pm$0.26$\%$          \\ \cline{1-1} \cline{3-8} 
DLM             &                                  & \multicolumn{1}{c|}{6.219$\%\pm$1.59$\%$}           & \multicolumn{1}{c|}{-0.485$\%\pm$0.95$\%$}          & -8.733$\%\pm$0.43$\%$          & \multicolumn{1}{c|}{-5.317$\%\pm$0.84$\%$}         & \multicolumn{1}{c|}{-6.454$\%\pm$0.39$\%$}          & -7.674$\%\pm$0.19$\%$          \\ \cline{1-1} \cline{3-8} 
DLM-PromptEngg  &                                  & \multicolumn{1}{c|}{4.341$\%\pm$1.65$\%$}           & \multicolumn{1}{c|}{-3.57$\%\pm$0.91$\%$}           & -9.417$\%\pm$0.57$\%$          & \multicolumn{1}{c|}{-4.957$\%\pm$0.94$\%$}         & \multicolumn{1}{c|}{-7.366$\%\pm$0.5$\%$}           & -7.586$\%\pm$0.37$\%$          \\ \hline
SCLM-SIM        & \multirow{2}{*}{Utilitarian}     & \multicolumn{1}{c|}{\textbf{21.502$\%\pm$1.76$\%$}} & \multicolumn{1}{c|}{\textbf{10.643$\%\pm$1.06$\%$}} & \textbf{-3.434$\%\pm$0.53$\%$} & \multicolumn{1}{c|}{\textbf{2.206$\%\pm$0.89$\%$}} & \multicolumn{1}{c|}{\textbf{-1.203$\%\pm$0.45$\%$}} & -4.423$\%\pm$0.29$\%$          \\ \cline{1-1} \cline{3-8} 
SCLM-LLM        &                                  & \multicolumn{1}{c|}{8.711$\%\pm$1.3$\%$}            & \multicolumn{1}{c|}{-0.574$\%\pm$1.08$\%$}          & -8.285$\%\pm$0.82$\%$          & \multicolumn{1}{c|}{-6.512$\%\pm$0.63$\%$}         & \multicolumn{1}{c|}{-6.457$\%\pm$0.55$\%$}          & -6.756$\%\pm$0.43$\%$          \\ \hline
SCLM-SIM        & \multirow{2}{*}{Egalitarian}     & \multicolumn{1}{c|}{11.481$\%\pm$1.61$\%$}          & \multicolumn{1}{c|}{7.92$\%\pm$1.01$\%$}            & \textbf{-0.416$\%\pm$0.37$\%$} & \multicolumn{1}{c|}{\textbf{2.109$\%\pm$0.75$\%$}} & \multicolumn{1}{c|}{\textbf{0.287$\%\pm$0.4$\%$}}   & \textbf{-1.974$\%\pm$0.22$\%$} \\ \cline{1-1} \cline{3-8} 
SCLM-LLM        &                                  & \multicolumn{1}{c|}{8.239$\%\pm$1.67$\%$}           & \multicolumn{1}{c|}{-1.919$\%\pm$0.98$\%$}          & -9.488$\%\pm$0.48$\%$          & \multicolumn{1}{c|}{-1.561$\%\pm$0.8$\%$}          & \multicolumn{1}{c|}{-4.67$\%\pm$0.39$\%$}           & -7.059$\%\pm$0.2$\%$           \\ \hline
SCLM-SIM        & \multirow{2}{*}{Nash}            & \multicolumn{1}{c|}{\textbf{20.579$\%\pm$1.8$\%$}}  & \multicolumn{1}{c|}{\textbf{10.195$\%\pm$1.1$\%$}}  & -4.376$\%\pm$0.54$\%$          & \multicolumn{1}{c|}{2.09$\%\pm$0.9$\%$}            & \multicolumn{1}{c|}{-1.568$\%\pm$0.47$\%$}          & -4.886$\%\pm$0.29$\%$          \\ \cline{1-1} \cline{3-8} 
SCLM-LLM        &                                  & \multicolumn{1}{c|}{-0.732$\%\pm$1.25$\%$}          & \multicolumn{1}{c|}{-4.854$\%\pm$0.95$\%$}          & -10.455$\%\pm$0.68$\%$         & \multicolumn{1}{c|}{-9.877$\%\pm$0.65$\%$}         & \multicolumn{1}{c|}{-8.373$\%\pm$0.53$\%$}          & -8.107$\%\pm$0.36$\%$          \\ \hline
\end{tabular}
\label{app:full_table_rw}
}
\end{table*}

\begin{table*}
\centering
\renewcommand{\arraystretch}{1.3}
\caption{\label{app:trad}
\color{black}
Results comparing different reward function choice selection strategies for deciding tradeoff between preference alignment, unintended shift minimization and utility maximization, aggregated across the three synthetic datasets (left) and three real world datasets (right). 
A higher summed \% change in desired features implies a better alignment, whereas less unintended shift is better.
\color{black}
}
\resizebox{0.95\textwidth}{!}{
\begin{tabular}{|c|l|cc|c|cc|}
\cline{1-1} \cline{3-4} \cline{6-7}
                        &  & \multicolumn{2}{c|}{Real World}                                                                                                                                                      &  & \multicolumn{2}{c|}{Synthetic}                                                                                                                                                       \\ \cline{1-1} \cline{3-4} \cline{6-7} 
\textbf{Method}         &  & \multicolumn{1}{c|}{\textbf{\begin{tabular}[c]{@{}c@{}}Summed \% Change in\\ Desired Feature(s)\end{tabular}}} & \textbf{\begin{tabular}[c]{@{}c@{}}Unintended\\ Shift\end{tabular}} &  & \multicolumn{1}{c|}{\textbf{\begin{tabular}[c]{@{}c@{}}Summed \% Change in\\ Desired Feature(s)\end{tabular}}} & \textbf{\begin{tabular}[c]{@{}c@{}}Unintended\\ Shift\end{tabular}} \\ \cline{1-1} \cline{3-4} \cline{6-7} 
DLM-PrioritizationOnly  &  & \multicolumn{1}{c|}{6.809$\pm$0.86}                                                                            & 0.302$\pm$0.02                                                      &  & \multicolumn{1}{c|}{0.06$\pm$1.13}                                                                             & 1.15$\pm$0.04                                                       \\ \cline{1-1} \cline{3-4} \cline{6-7} 
DLM-ExtendedPrompt-Fair &  & \multicolumn{1}{c|}{-0.254$\pm$0.73}                                                                           & 0.276$\pm$0.02                                                      &  & \multicolumn{1}{c|}{-1.882$\pm$1.31}                                                                           & 1.173$\pm$0.06                                                      \\ \cline{1-1} \cline{3-4} \cline{6-7} 
SCLM-PrioritizationOnly                    &  & \multicolumn{1}{c|}{13.131$\pm$0.86}                                                                           & 0.316$\pm$0.02                                                      &  & \multicolumn{1}{c|}{9.349$\pm$1.45}                                                                            & 1.218$\pm$0.05                                                      \\ \cline{1-1} \cline{3-4} \cline{6-7} 
SCLM-ExtendedPrompt-Fair  &  & \multicolumn{1}{c|}{15.364$\pm$0.94}                                                                           & 0.099$\pm$0.01                                                      &  & \multicolumn{1}{c|}{10.006$\pm$1.13}                                                                           & 0.239$\pm$0.02                                                      \\ \cline{1-1} \cline{3-4} \cline{6-7} 
\end{tabular}
}
\label{tab:shifts}
\end{table*}

\begin{table*}
\centering
\renewcommand{\arraystretch}{1.3}
\caption{
\color{black}
Results comparing different reward function choice selection strategies for deciding tradeoff between preference alignment, unintended shift minimization and utility maximization, aggregated across the three synthetic datasets (left) and three real world datasets (right). 
A higher summed \% change in desired features implies a better alignment, whereas less unintended shift is better.
\color{black}
}
\resizebox{0.95\textwidth}{!}{
\begin{tabular}{|c|l|ccc|c|ccc|}
\cline{1-1} \cline{3-5} \cline{7-9}
                        &  & \multicolumn{3}{c|}{Real World}                                                                                                                                                                                                                                                           &  & \multicolumn{3}{c|}{Synthetic}                                                                                                                                                                                                                                                            \\ \cline{1-1} \cline{3-5} \cline{7-9} 
\textbf{Method}         &  & \multicolumn{1}{c|}{\textbf{\begin{tabular}[c]{@{}c@{}}Summed \% Change in\\ Desired Feature(s)\end{tabular}}} & \multicolumn{1}{c|}{\textbf{\begin{tabular}[c]{@{}c@{}}Normalized\\ Utility Score\end{tabular}}} & \textbf{\begin{tabular}[c]{@{}c@{}}\% Drop in\\ Utility\end{tabular}} &  & \multicolumn{1}{c|}{\textbf{\begin{tabular}[c]{@{}c@{}}Summed \% Change in\\ Desired Feature(s)\end{tabular}}} & \multicolumn{1}{c|}{\textbf{\begin{tabular}[c]{@{}c@{}}Normalized\\ Utility Score\end{tabular}}} & \textbf{\begin{tabular}[c]{@{}c@{}}\% Drop in\\ Utility\end{tabular}} \\ \cline{1-1} \cline{3-5} \cline{7-9} 
DLM-PrioritizationOnly                     &  & \multicolumn{1}{c|}{6.809$\pm$0.86}                                                                            & \multicolumn{1}{c|}{0.317$\pm$0.01}                                                              & -6.233$\pm$0.17                                                       &  & \multicolumn{1}{c|}{0.06$\pm$1.13}                                                                             & \multicolumn{1}{c|}{0.485$\pm$0.01}                                                              & -5.403$\pm$0.28                                                       \\ \cline{1-1} \cline{3-5} \cline{7-9} 
DLM-ExtendedPrompt-Util &  & \multicolumn{1}{c|}{3.416$\pm$1.04}                                                                            & \multicolumn{1}{c|}{0.347$\pm$0.02}                                                              & -5.689$\pm$0.25                                                       &  & \multicolumn{1}{c|}{-2.508$\pm$1.2}                                                                            & \multicolumn{1}{c|}{0.548$\pm$0.02}                                                              & -3.825$\pm$0.29                                                       \\ \cline{1-1} \cline{3-5} \cline{7-9} 
SCLM-PrioritizationOnly                    &  & \multicolumn{1}{c|}{13.131$\pm$0.86}                                                                           & \multicolumn{1}{c|}{0.397$\pm$0.01}                                                              & -5.372$\pm$0.19                                                       &  & \multicolumn{1}{c|}{9.349$\pm$1.45}                                                                            & \multicolumn{1}{c|}{0.614$\pm$0.01}                                                              & -2.8$\pm$0.24                                                         \\ \cline{1-1} \cline{3-5} \cline{7-9} 
SCLM-ExtendedPrompt-Util      &  & \multicolumn{1}{c|}{8.869$\pm$0.73}                                                                            & \multicolumn{1}{c|}{0.519$\pm$0.02}                                                              & -3.91$\pm$0.21                                                        &  & \multicolumn{1}{c|}{13.211$\pm$1.31}                                                                           & \multicolumn{1}{c|}{0.736$\pm$0.01}                                                              & -0.568$\pm$0.12                                                       \\ \cline{1-1} \cline{3-5} \cline{7-9} 
SCLM-Full  &  & \multicolumn{1}{c|}{11.496$\pm$0.87}                                                                           & \multicolumn{1}{c|}{0.529$\pm$0.01}                                                              & -3.85$\pm$0.18                                                        &  & \multicolumn{1}{c|}{9.166$\pm$1.05}                                                                            & \multicolumn{1}{c|}{0.705$\pm$0.01}                                                              & -1.044$\pm$0.11                                                       \\ \cline{1-1} \cline{3-5} \cline{7-9} 
\end{tabular}
}
\label{tab:shifts}
\end{table*}

\subsection{Effect of Social Choice Function}\label{app:effect_swf}

\begin{figure}[h!]
    \centering
        \includegraphics[width=0.5\textwidth]{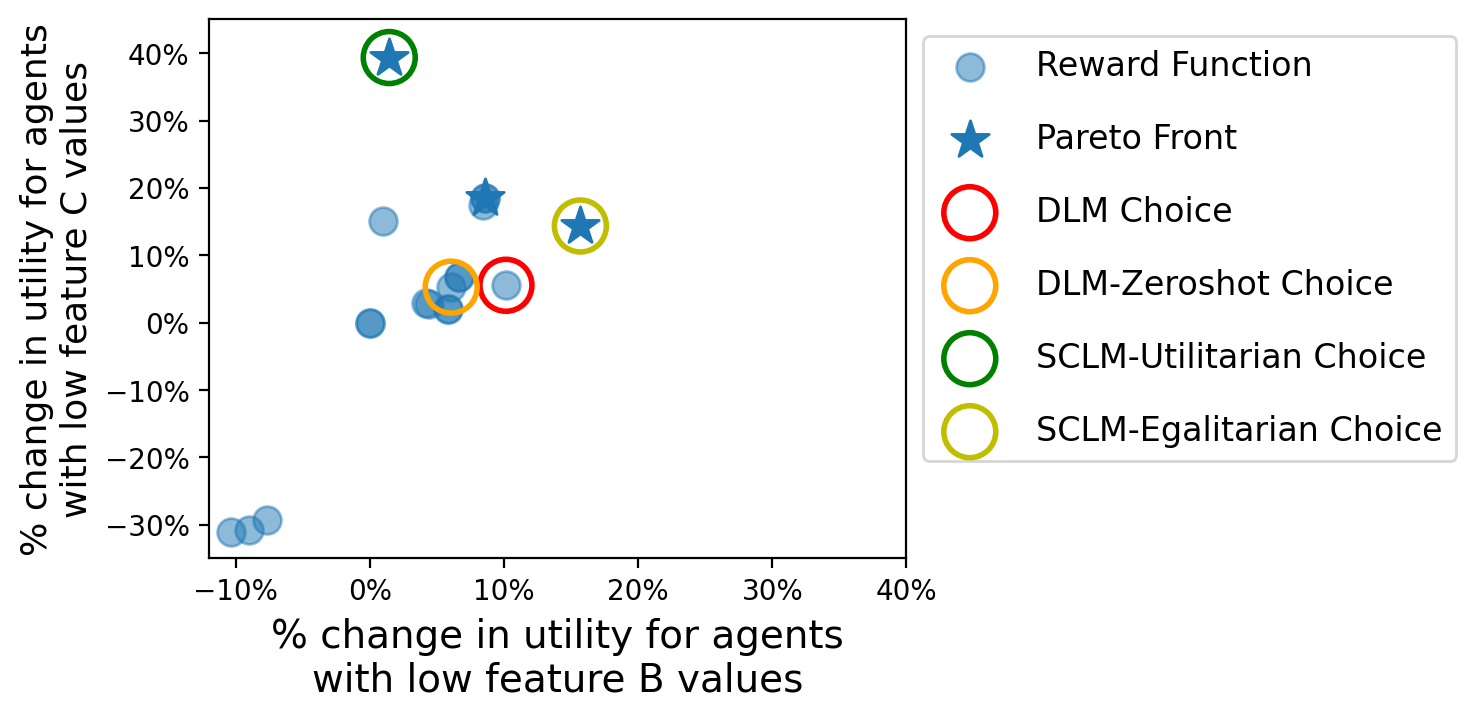}
        \caption{Comparison of different methods using the prompt ``Prioritize agents with a low value for feature $B$ and agents with a low value for feature $C$''.}
        \label{fig:example_result}
\end{figure}   
While both the Utilitarian and Egalitarian social welfare function leads to good overall results, the choice between them significantly influences which reward function is selected at the instance level. 
\Cref{fig:example_result} shows one example instance for a prompt with two prioritization clauses. As in Figure 1, each point corresponds to a candidate reward function with the axes measuring alignment with the two respective clauses. (as described in the beginning of the section 6.3). 
The SCLM-Utilitarian model chooses a reward function from the Pareto frontier that shows a much stronger effect for the second clause (i.e., the utility increase for agents with a low value of feature $C$ is more pronounced). 
In contrast, the reward function selected by
SCLM-Egalitarian is much more balanced. DLM-based baselines fail to pick a reward function from the Pareto frontier.

\subsection{Preventing Ineffective Resource Use}
We additionally consider situations where the user specifies prioritization prompts for multiple attributes but is still interested in maximizing the total generated utility.
As discussed at the end of \Cref{ad}, our framework can handle such settings by adding utility maximization as an additional evaluation dimension. For all 6 singular and 12 composite prioritization prompts in our tests, we add additional clauses to maximize total utility and call the resulting model SCLM-ExtendedPrompt-Util. As a baseline, we consider three models: DLM/SCLM prompted only by the prioritization clause (DLM/SCLM-Prioritization Only) and DLM prompted by both the prioritization clause(s) and a clause requesting the maximization of overall utility (DLM-ExtendedPrompt-Util).

As before, to assess alignment with prioritization clauses, we calculate the average change in utility for the prioritized subgroup. To quantify the performance of a reward function with regards to maximizing utility, we
calculate the percent drop in utility as compared to the default reward function.

Table 6 shows the results. 
Comparing \emph{DLM-PrioritizationOnly} and \emph{DLM-ExtendedPrompt-Util}, we again observe that adding the additional objective to the prompt does not result in a better performance in real-world domain and synthetic domain. 
In contrast, SCLM-ExtendedPrompt-Util which incorporates utility as a scoring dimension chooses reward functions resulting in a significantly higher utility increases for the prioritized subpopulations and a significantly higher total utility (i.e., there is a lower drop in utility).

\paragraph{Combining Prioritization, Unintended Shifts, and Utility Drops}
Finally, we consider a variant of SCLM which both prevents unintended shifts and utility drops, i.e., in addition to the specified prioritization clauses, we add clause(s) on the minimization of shifts and a clause on the minimization of utility drops (which results in having a total of 4 evaluation dimensions; one per feature plus the utility drop one). We use the simulator scorer together with the utilitarian welfare function and call the resulting model SCLM-Full. We see in the last row of \Cref{tab:shifts} that this model chooses balanced reward functions with high \% change in utility in desired attributes, high overall utility, and low unintended shift, showcasing that SCLM can flexibly allow the end-user to combine multiple evaluation dimensions.

\color{black}

\section{Prompt Texts}\label{app:prompt}
In Figures \ref{app:prompt_gen_synth}-\ref{app:prompt_choose_EP_util_rw_contd}, we show the prompts passed to LLM for various experiments described in the paper. Specifically, prompt in Figures \ref{app:prompt_gen_synth},\ref{app:prompt_gen_rw} show how to generate a reward function in code form based on the problem description, the indices of features, and preference goals for synthetic and real-world problem domains, respectively.  Figures \ref{app:prompt_choose_synth},
\ref{app:prompt_choose_rw} show how the LLM is prompted to select the best reward function from those generated in the previous step for synthetic and real-world domains, respectively. Together, these prompts establish the DLM baseline in the paper.
LLM-Zeroshot baseline has the exact same prompt as Figures \ref{app:prompt_gen_synth}, \ref{app:prompt_gen_rw} for synthetic and real world domains. The only difference is that there is no reflection step, and the first reward generated by LLM is chosen as the best reward.

In Figure \ref{app:prompt_choose_EP_rw} and \ref{app:prompt_choose_EP_rw_contd}, we show how the prompt is enhanced with additional information to assist in selecting a reward function. The additional information is highlighted in bold.
Section 6.4 explains how SCLM is used to specify additional objectives, such as minimizing utility shifts in features not included in the preference prompt and maximizing overall utility. In Figures \ref{app:prompt_choose_EP_shift_rw}, \ref{app:prompt_choose_EP_shift_rw_contd} we show the prompts for the DLM-EP baselines, which explicitly include instructions to minimize utility shifts in unintended features. In Figures \ref{app:prompt_choose_EP_util_rw}, \ref{app:prompt_choose_EP_util_rw_contd}, we show the prompts that explicitly include instructions for maximizing overall utility.
\begin{figure*}[ht]
\centering
\begin{conversation}[label={box:rw_generator}]{Generator: Synthetic Domain}{}
\vspace{0.5em}
\large \textbf{Prompt} \\
\small
Create a Python reward function for RL in resource allocation problem for agents, with the objective of prioritizing higher states and: Focus on the agents with low 
value of feature A.. The function should use 'state' (value is either 0,1) and features 'agent\_feats' (length 15 array) to direct the RL agent. Here is a description
of the features you may use along with the index in 'agent\_feats' array: 
\newline
Index Name DataType\\
 0. Feature A bucket 1 - Binary\\
 1. Feature A bucket 2 - Binary\\
 2. Feature A bucket 3 - Binary\\
 3. Feature A bucket 4 - Binary\\
 4. Feature A bucket 5 - Binary\\
 5. Feature B bucket 1 - Binary\\
 6. Feature B bucket 2 - Binary\\
 7. Feature B bucket 3 - Binary\\
 8. Feature B bucket 4 - Binary\\
 9. Feature B bucket 5 - Binary\\
 10. Feature C bucket 1 - Binary\\
 11. Feature C bucket 2 - Binary\\
 12. Feature C bucket 3 - Binary\\
 13. Feature C bucket 4 - Binary\\
 14. Feature C bucket 5 - Binary\\
All buckets are in increasing order of the feature values. For example, 'Feature A bucket 1' would consist of bucket of lowest values of feature A while 'Feature A 
bucket 5' would consist of highest value of feature A. This is true for Feature A, B and C.\\
\newline
Your task:
1. Write a simple, single-line Python reward function. Exclude the word 'return' and exclude non-standard libraries. Format your code with triple \$ signs: \$\$\$[YOUR 
FUNCTION]\$\$\$. 
Note that HIGHER states are always preferred, so ensure reward increases as state increases. Make sure reward is always positive and increasing with state. \\
\newline
Example Prompt: Prioritize agents that have low feature A and high feature C\\

Example Response:
Python Code: '\$\$\$ state+state * ((agent\_feats[0] or agent\_feats[1]) and (agent\_feats[17] or agent\_feats[18] or agent\_feats[19])) \$\$\$'
or '\$\$\$ state * (agent\_feats[0] or 3*agent\_feats[19]) \$\$\$'
or '\$\$\$ state + 2*state * ((5*agent\_feats[0]+agent\_feats[1]) and agent\_feats[19]) \$\$\$'
In these example, agent\_feats[0] and agent\_feats[1] represent agents with low values for feature A and agent\_feats[17], agent\_feats[18], agent\_feats[19] represent 
agents with high values for feature C\\
It is upto you to decide which features will represent a preference. For example low values could be the lowest feature bucket, or lower three feature buckets or so 
on. Come up with a unique new reward for the specified goal: Focus on the agents with low value of feature A..
Goal: 
Focus on the agents with low value of feature A.

\large \textbf{Output}\\
\small
\$\$\$ 2*state + state * (1*agent\_feats[0]+ 0.5*agent\_feats[1]) \$\$\$
\end{conversation}
\caption{Prompt passed to the LLM to generate reward function based on the problem scenarios in the Synthetic Domain.}
\label{app:prompt_gen_synth}
\end{figure*}

\begin{figure*}[ht]
\centering
\begin{conversation}[label={box:rw_generator}]{Generator: Real World Domain}{}
\vspace{0.5em}
\large \textbf{Prompt} \\
\small
Create a Python reward function for RL in phone call resource allocation to mothers in India, with the objective of prioritizing higher states and: Focus on the 
young mothers by age. The function should use 'state' (value is either 0,1) and features agent feats (length 43 array) to direct the RL agent. Here is a 
description of the features you may use along with the index in agent feats array: \\
\newline
Index Name DataType\\
 0. Ages 10-20 - Binary\\
 1. Ages 21-30 - Binary\\
 2. Ages 31-40 - Binary\\
 3. Ages 41-50 - Binary\\
 4. Ages 51-60 - Binary\\
 5. Education level 1/7 -- illiterate - Binary\\
 6. Education level 2/7 -- 1-5th Grade Completed - Binary\\
 7. Education level 3/7 -- 6-9th Grade Completed - Binary\\
 8. Education level 4/7 -- 10th Grade Passed - Binary\\
 9. Education level 5/7 -- 12th Grade Passed - Binary\\
 10. Education level 6/7 -- Graduate - Binary\\
 11. Education level 7/7 -- Post graduate - Binary\\
 12. Income bracket 1 (e.g., 0-5000) - Binary\\
 13. Income bracket 2 (e.g., 5001-10000) - Binary\\
 14. Income bracket 3 (e.g., 10001-15000) - Binary\\
 15. Income bracket 4 (e.g., 15001-20000) - Binary\\
 16. Income bracket 5 (e.g., 20001-25000) - Binary\\
 17. Income bracket 6 (e.g., 25001-30000) - Binary\\
 18. Income bracket 7 (e.g., 30000-999999) - Binary\\
 \newline
\textbf{Your task:}\\
Write a simple, single-line Python reward function. Exclude the word 'return' and exclude non-standard libraries. Format your code with triple \$ signs: \$\$\$[YOUR 
FUNCTION]\$\$\$. 
Note that HIGHER states are always preferred, so ensure reward increases as state increases. Make sure reward is always positive and increasing with state. 
Example Prompt: Prioritize agents that are older and rich\\
Example Response:\\
Python Code: '\$\$\$ state * (agent\_feats[4] and agent\_feats[18]) \$\$\$'\\
Come up with a unique new reward for the specified goal: Focus on the young mothers by age.
Goal: 
Focus on the young mothers by age\\
\newline
\large \textbf{Output}\\
\small
\$\$\$ state * agent\_feats[7] \$\$\$
\end{conversation}
\caption{Prompt passed to the LLM to generate reward function based on the problem scenarios in the Real World Domain.}
\label{app:prompt_gen_rw}
\end{figure*}

\begin{figure*}[ht]
\centering
\begin{conversation}[label={box:rw_generator}]{DLM Choice: Synthetic Domain}{}
\vspace{0.5em}
\large \textbf{Prompt} \\
\small
My goal was to create a Python reward function for RL in resource allocation, with the objective of: Focus on the agents with low value of feature A. I tried several
reward functions for this task. 

Below are the reward functions I used and their corresponding reward distributions:

\textbf{Function Number 0: }
Reward Function:  2*state + state * (1*agent\_feats[0]+ 0.5*agent\_feats[1])  
\newline
Reflection:
 '

Category: A
Feature A bucket 1: 113.20\\
Feature A bucket 2: 137.07\\
Feature A bucket 3: 56.51\\
Feature A bucket 4: 56.82\\
Feature A bucket 5: 54.60\\

Category: B
Feature B bucket 1: 82.89\\
Feature B bucket 2: 65.36\\
Feature B bucket 3: 60.33\\
Feature B bucket 4: 50.13\\
Feature B bucket 5: 46.00\\

Category: C
Feature C Bucket 1: 46.78\\
Feature C Bucket 2: 49.96\\
Feature C Bucket 3: 46.64\\
Feature C Bucket 4: 66.58\\
Feature C Bucket 5: 62.24\\

\textbf{Function Number 1: }
Reward Function:  state * (agent\_feats[0] and not (agent\_feats[1] or agent\_feats[2]))   
\newline
Reflection:
 '
Category: A
Feature A bucket 1: 177.84\\
Feature A bucket 2: 54.76\\
Feature A bucket 3: 55.93\\
Feature A bucket 4: 57.00\\
Feature A bucket 5: 55.29\\

Category: B
Feature B bucket 1: 64.47\\
Feature B bucket 2: 58.76\\
Feature B bucket 3: 56.96\\
Feature B bucket 4: 51.11\\
Feature B bucket 5: 50.13\\

Category: C
Feature C Bucket 1: 49.80\\
Feature C Bucket 2: 51.69\\
Feature C Bucket 3: 48.40\\
Feature C Bucket 4: 65.89\\
Feature C Bucket 5: 60.00\\

Based on the above reward distributions and the given goal: Focus on those with high education., please identify the FUNCTION NUMBER of the most effective reward 
function. Provide your answer EXACTLY IN the following format: 'The best reward function is at number: [FUNCTION NUMBER]'.\\
\newline
\large \textbf{Output: }
\newline
\small
The best reward function is at number: 1
\end{conversation}
\caption{Prompt passed to the LLM to choose a reward function based on the context of problem scenario in the Synthetic Domain, the generated reward functions and the reward distribution corresponding to every reward function.}
\label{app:prompt_choose_synth}
\end{figure*}

\begin{figure*}[ht]
\centering
\begin{conversation}[label={box:rw_generator}]{DLM Choice: Real World Domain}{}
\vspace{0.5em}
\large \textbf{Prompt} \\
\small
My goal was to create a Python reward function for RL in resource allocation, with the objective of: Focus on those with high education. I tried several reward 
functions for this task. Below, I have the given reward function, and the corresponding distribution of reward achieved across 20 agent features. 

Below are the reward functions I used and their corresponding reward distributions:

\textbf{Function Number 0: }
Reward Function:  -agent\_feats[5] -agent\_feats[6]-agent\_feats[7]-agent\_feats[8]-agent\_feats[9]-agent\_feats[10]-agent\_feats[11]  
\newline
Reflection:
 '

Category: Age
Ages 10-20: 121.73\\
Ages 21-30: 421.04\\
Ages 31-40: 244.49\\
Ages 41-50: 64.11\\
Ages 51-60: 10.58\\

Category: Income
Income bracket 1 (e.g., 0-5000): 126.82\\
Income bracket 2 (e.g., 5001-10000): 373.62\\
Income bracket 3 (e.g., 10001-15000): 234.87\\
Income bracket 4 (e.g., 15001-20000): 77.40\\
Income bracket 5 (e.g., 20001-25000): 35.58\\
Income bracket 6 (e.g., 25001-30000): 2.58\\
Income bracket 7 (e.g., 30000-999999): 11.09\\

Category: Education
Illiterate: 39.91\\
1-5th Grade Completed: 157.84\\
6-9th Grade Completed: 281.36\\
10th Grade Passed: 197.64\\
12th Grade Passed: 103.18\\
Graduate: 21.13\\
Post graduate: 60.89\\

\textbf{Function Number 1: }
Reward Function:  state * agent\_feats[10]  
\newline
Reflection:
 '
Category: Age
Ages 10-20: 134.22\\
Ages 21-30: 469.16\\
Ages 31-40: 270.44\\
Ages 41-50: 72.80\\
Ages 51-60: 11.96\\

Category: Income
Income bracket 1 (e.g., 0-5000): 138.40\\
Income bracket 2 (e.g., 5001-10000): 414.44\\
Income bracket 3 (e.g., 10001-15000): 266.44\\
Income bracket 4 (e.g., 15001-20000): 85.33\\
Income bracket 5 (e.g., 20001-25000): 40.20\\
Income bracket 6 (e.g., 25001-30000): 2.80\\
Income bracket 7 (e.g., 30000-999999): 10.96\\

Category: Education
Illiterate: 45.07\\
1-5th Grade Completed: 173.82\\
6-9th Grade Completed: 314.07\\
10th Grade Passed: 217.31\\
12th Grade Passed: 113.02\\
Graduate: 29.36\\
Post graduate: 65.93\\

Based on the above reward distributions and the given goal: Focus on those with high education., please identify the FUNCTION NUMBER of the most effective reward 
function. Provide your answer EXACTLY IN the following format: 'The best reward function is at number: [FUNCTION NUMBER]'.\\
\newline
\large \textbf{Output: }
\newline
\small
The best reward function is at number: 1
\end{conversation}
\caption{Prompt passed to the LLM to choose a reward function based on the context of problem scenario in Real World Domain, the generated reward functions and the reward distribution corresponding to every reward function.}
\label{app:prompt_choose_rw}
\end{figure*}

\begin{figure*}[ht]
\centering
\begin{conversation}[label={box:rw_generator}]{DLM Choice with Prompt Engineering (DLM-PromptEngg): Real World Domain}{}
\vspace{0.5em}
\large \textbf{Prompt} \\
\small
My goal was to create a Python reward function for RL in resource allocation, with the objective of: Focus on those with high education. I tried several reward 
functions for this task. Below, I have the given reward function, and the corresponding distribution of reward achieved across 20 agent features. 

Below are the reward functions I used and their corresponding reward distributions:

\textbf{Function Number 0: }
Reward Function:  -agent\_feats[5] -agent\_feats[6]-agent\_feats[7]-agent\_feats[8]-agent\_feats[9]-agent\_feats[10]-agent\_feats[11]  
\newline
Reflection:
 '

Category: Age
Ages 10-20: 121.73\\
Ages 21-30: 421.04\\
Ages 31-40: 244.49\\
Ages 41-50: 64.11\\
Ages 51-60: 10.58\\

Category: Income
Income bracket 1 (e.g., 0-5000): 126.82\\
Income bracket 2 (e.g., 5001-10000): 373.62\\
Income bracket 3 (e.g., 10001-15000): 234.87\\
Income bracket 4 (e.g., 15001-20000): 77.40\\
Income bracket 5 (e.g., 20001-25000): 35.58\\
Income bracket 6 (e.g., 25001-30000): 2.58\\
Income bracket 7 (e.g., 30000-999999): 11.09\\

Category: Education
Illiterate: 39.91\\
1-5th Grade Completed: 157.84\\
6-9th Grade Completed: 281.36\\
10th Grade Passed: 197.64\\
12th Grade Passed: 103.18\\
Graduate: 21.13\\
Post graduate: 60.89\\

\textbf{Function Number 1: }
Reward Function:  state * agent\_feats[10]  
\newline
Reflection:
 '
Category: Age
Ages 10-20: 134.22\\
Ages 21-30: 469.16\\
Ages 31-40: 270.44\\
Ages 41-50: 72.80\\
Ages 51-60: 11.96\\

Category: Income
Income bracket 1 (e.g., 0-5000): 138.40\\
Income bracket 2 (e.g., 5001-10000): 414.44\\
Income bracket 3 (e.g., 10001-15000): 266.44\\
Income bracket 4 (e.g., 15001-20000): 85.33\\
Income bracket 5 (e.g., 20001-25000): 40.20\\
Income bracket 6 (e.g., 25001-30000): 2.80\\
Income bracket 7 (e.g., 30000-999999): 10.96\\

Category: Education
Illiterate: 45.07\\
1-5th Grade Completed: 173.82\\
6-9th Grade Completed: 314.07\\
10th Grade Passed: 217.31\\
12th Grade Passed: 113.02\\
Graduate: 29.36\\
Post graduate: 65.93\\

\end{conversation}
\caption{Enhanced Prompt passed to the LLM to choose a reward function based on the context of the problem scenario in the Real World Domain, the generated reward functions, the reward distributions corresponding to every reward function and additional examples on what to look at when choosing a reward function aligned with the preference.}
\label{app:prompt_choose_EP_rw}
\end{figure*}

\begin{figure*}[ht]
\centering
\begin{conversation}{}
\vspace{0.5em}

Based on the above reward distributions and the given goal: Focus on the those with low education., please identify the FUNCTION NUMBER of the most effective reward 
function. 
\textbf{You can look at the reward distributions for different features and based on them, judge the effectiveness of the correponding reward function.
For instance, if the query wants to prioritize low income agents, you should look if the rewards are indeed high for low income features. it is upto you to decide which features describe low income preference.}
Provide your answer EXACTLY IN the following format: 'The best reward 
function is at number: [FUNCTION NUMBER]'..\\
\newline
\large \textbf{Output: }
\newline
\small
The best reward function is at number: 1
\end{conversation}
\caption{Enhanced Prompt passed to the LLM to choose a reward function based on the context of the problem scenario in the Real World Domain, the generated reward functions, the reward distributions corresponding to every reward function and additional examples on what to look at when choosing a reward function aligned with the preference.}
\label{app:prompt_choose_EP_rw_contd}
\end{figure*}

\begin{figure*}[ht]
\centering
\begin{conversation}[label={box:rw_generator}]{DLM Choice with Extended Prompt for Minimizing Utility Shifts (DLM-EP): Real World Domain}{}
\vspace{0.5em}
\large \textbf{Prompt} \\
\small
My goal was to create a Python reward function for RL in resource allocation, with the objective of: \textbf{Focus on the young mothers by age and also focus on those with low education}. I tried several reward 
functions for this task. Below, I have the given reward function, and the corresponding distribution of reward achieved across 20 agent features. 

Below are the reward functions I used and their corresponding reward distributions:\\
\newline
\textbf{Function Number 0: }
Reward Function: state * (agent\_feats[0] or agent\_feats[1]) and (agent\_feats[5] or agent\_feats[6])    
\newline
Reflection:
 '

Category: Age
Ages 10-20: 163.24\\
Ages 21-30: 547.98\\
Ages 31-40: 269.78\\
Ages 41-50: 72.11\\
Ages 51-60: 10.91\\

Category: Income
Income bracket 1 (e.g., 0-5000): 154.40\\
Income bracket 2 (e.g., 5001-10000): 472.98\\
Income bracket 3 (e.g., 10001-15000): 293.53\\
Income bracket 4 (e.g., 15001-20000): 89.82\\
Income bracket 5 (e.g., 20001-25000): 40.84\\
Income bracket 6 (e.g., 25001-30000): 2.91\\
Income bracket 7 (e.g., 30000-999999): 9.53\\

Category: Education
Illiterate: 66.47\\
1-5th Grade Completed: 257.87\\
6-9th Grade Completed: 312.69\\
10th Grade Passed: 224.22\\
12th Grade Passed: 113.53\\
Graduate: 23.42\\
Post graduate: 65.82\\

\textbf{Function Number 1: }
Reward Function:  state * (agent\_feats[0] or agent\_feats[1]) * (agent\_feats[5] or agent\_feats[6])   
\newline
Reflection:
 '
Category: Age
Ages 10-20: 163.24\\
Ages 21-30: 547.98\\
Ages 31-40: 269.78\\
Ages 41-50: 72.11\\
Ages 51-60: 10.91\\

Category: Income
Income bracket 1 (e.g., 0-5000): 154.40\\
Income bracket 2 (e.g., 5001-10000): 472.98\\
Income bracket 3 (e.g., 10001-15000): 293.53\\
Income bracket 4 (e.g., 15001-20000): 89.82\\
Income bracket 5 (e.g., 20001-25000): 40.84\\
Income bracket 6 (e.g., 25001-30000): 2.91\\
Income bracket 7 (e.g., 30000-999999): 9.53\\

Category: Education
Illiterate: 66.47\\
1-5th Grade Completed: 257.87\\
6-9th Grade Completed: 312.69\\
10th Grade Passed: 224.22\\
12th Grade Passed: 113.53\\
Graduate: 23.42\\
Post graduate: 65.82\\
\textbf{Additional Information - Rewards from Default reward function} \textit{(Reward distribution from Default reward function. Truncated for brevity.)}

\end{conversation}
\caption{Enhanced Prompt passed to the LLM to choose a reward function based on the context of the problem scenario in the Real World Domain, the generated reward functions, the reward distributions corresponding to every reward function and additional information to minimize the unintended utility shifts in dimensions not specified in the preference.}
\label{app:prompt_choose_EP_shift_rw}
\end{figure*}

\begin{figure*}[ht]
\centering
\begin{conversation}{}
\vspace{0.5em}

Based on the above reward distributions and the given goal: \textbf{Focus on the young mothers by age and also focus on those with low education}, please identify the FUNCTION NUMBER of the most effective reward 
function. 
\textbf{Also make sure that that you choose a reward function that does not cause unintended shifts in reward. Unintended shifts in reward here means that the chosen reward function shouldn't drastically change the distribution
in reward with respect to features not specified in the prompt
For example, if the prompt is to prefer agents with low education, then the chosen reward function shouldn't change the distribution in reward w.r.t the default 
reward distribution too much in the income feature buckets
}.
Provide your answer EXACTLY IN the following format: 'The best reward 
function is at number: [FUNCTION NUMBER]'.\\
\newline
\large \textbf{Output: }
\newline
\small
The best reward function is at number: 1
\end{conversation}
\caption{Continued: Enhanced Prompt passed to the LLM to choose a reward function based on the context of the problem scenario in the Real World Domain, the generated reward functions, the reward distributions corresponding to every reward function and additional information to minimize the unintended utility shifts in dimensions not specified in the preference.}
\label{app:prompt_choose_EP_shift_rw_contd}
\end{figure*}

\begin{figure*}[ht]
\centering
\begin{conversation}[label={box:rw_generator}]{DLM Choice with Extended Prompt for Maximizing Overall Utility (DLM-EP): Real World Domain}{}
\vspace{0.5em}
\large \textbf{Prompt} \\
\small
My goal was to create a Python reward function for RL in resource allocation, with the objective of: \textbf{Focus on the young mothers by age and also focus on those with low education}. I tried several reward 
functions for this task. Below, I have the given reward function, and the corresponding distribution of reward achieved across 20 agent features. 

Below are the reward functions I used and their corresponding reward distributions:\\
\newline
\textbf{Function Number 0: }
Reward Function: state * (agent\_feats[0] or agent\_feats[1]) and (agent\_feats[5] or agent\_feats[6])    
\newline
Reflection:
 '

Category: Age
Ages 10-20: 163.24\\
Ages 21-30: 547.98\\
Ages 31-40: 269.78\\
Ages 41-50: 72.11\\
Ages 51-60: 10.91\\

Category: Income
Income bracket 1 (e.g., 0-5000): 154.40\\
Income bracket 2 (e.g., 5001-10000): 472.98\\
Income bracket 3 (e.g., 10001-15000): 293.53\\
Income bracket 4 (e.g., 15001-20000): 89.82\\
Income bracket 5 (e.g., 20001-25000): 40.84\\
Income bracket 6 (e.g., 25001-30000): 2.91\\
Income bracket 7 (e.g., 30000-999999): 9.53\\

Category: Education
Illiterate: 66.47\\
1-5th Grade Completed: 257.87\\
6-9th Grade Completed: 312.69\\
10th Grade Passed: 224.22\\
12th Grade Passed: 113.53\\
Graduate: 23.42\\
Post graduate: 65.82\\

\textbf{Function Number 1: }
Reward Function:  state * (agent\_feats[0] or agent\_feats[1]) * (agent\_feats[5] or agent\_feats[6])    
\newline
Reflection:
 '
Category: Age
Ages 10-20: 163.24\\
Ages 21-30: 547.98\\
Ages 31-40: 269.78\\
Ages 41-50: 72.11\\
Ages 51-60: 10.91\\

Category: Income
Income bracket 1 (e.g., 0-5000): 154.40\\
Income bracket 2 (e.g., 5001-10000): 472.98\\
Income bracket 3 (e.g., 10001-15000): 293.53\\
Income bracket 4 (e.g., 15001-20000): 89.82\\
Income bracket 5 (e.g., 20001-25000): 40.84\\
Income bracket 6 (e.g., 25001-30000): 2.91\\
Income bracket 7 (e.g., 30000-999999): 9.53\\

Category: Education
Illiterate: 66.47\\
1-5th Grade Completed: 257.87\\
6-9th Grade Completed: 312.69\\
10th Grade Passed: 224.22\\
12th Grade Passed: 113.53\\
Graduate: 23.42\\
Post graduate: 65.82\\
\end{conversation}
\caption{Enhanced Prompt passed to the LLM to choose a reward function based on the context of the problem scenario in the Real World Domain, the generated reward functions, the reward distributions corresponding to every reward function and additional information to maximize the overall utility.}
\label{app:prompt_choose_EP_util_rw}
\end{figure*}

\begin{figure*}[ht]
\centering
\begin{conversation}{}
\vspace{0.5em}

Based on the above reward distributions and the given goal: \textbf{Focus on the young mothers by age and also focus on those with low education}, please identify the FUNCTION NUMBER of the most effective reward 
function. 
\textbf{Also make sure that that you choose a reward function which also maximizes the total reward. You can calculate this by adding up rewards in each feature 
bucket.}.
Provide your answer EXACTLY IN the following format: 'The best reward 
function is at number: [FUNCTION NUMBER]'.\\
\newline
\large \textbf{Output: }
\newline
\small
The best reward function is at number: 1
\end{conversation}
\caption{Continued: Enhanced Prompt passed to the LLM to choose a reward function based on the context of the problem scenario in the Real World Domain, the generated reward functions, the reward distributions corresponding to every reward function and additional information to maximize the overall utility.}
\label{app:prompt_choose_EP_util_rw_contd}
\end{figure*}

\end{document}